\documentclass{article}
\usepackage[top=4cm, bottom=4cm, left=3.1cm, right=3.1cm, includefoot]{geometry}

\usepackage{cite}
\usepackage{amsmath,amssymb,amsfonts}
\usepackage{accents}
\usepackage{graphicx}
\usepackage{textcomp}

\usepackage[whole]{bxcjkjatype}
\usepackage[normalem]{ulem}
\usepackage{epsfig, color}

\usepackage{caption}
\usepackage{amssymb}
\usepackage{amsmath}
\usepackage{amsthm}
\usepackage{multirow}
\usepackage{makecell}
\usepackage{placeins}
\usepackage{algorithm}
\usepackage{algpseudocode}
\newtheorem{thm}{Theorem}
\newtheorem{defn}{Definition}

\newtheorem{lem}{Lemma}
\newtheorem{rem}{Remark}
\newtheorem{example}{Example}

\begin{document}
\title{Spurious Vanishing Problem in Approximate Vanishing Ideal}
\author{Hiroshi Kera\footnote{Department of Information and Communication Engineering,
Graduate School of Information Science and Technology,
The University of Tokyo. Corresponding author: Hiroshi Kera (e-mail: kera@biom.t.u-tokyo.ac.jp).} \ and Yoshihiko Hasegawa$^*$}
\date{}

\maketitle

\begin{abstract}
Approximate vanishing ideal is a concept from computer algebra that studies the algebraic varieties behind perturbed data points. 
To capture the nonlinear structure of perturbed points, the introduction of approximation to exact vanishing ideals plays a critical role. However, such an approximation also gives rise to a theoretical problem---the spurious vanishing problem---in the basis construction of approximate vanishing ideals; namely, obtained basis polynomials can be approximately vanishing simply because of the small coefficients. 
In this paper, we propose a first general method that enables various basis construction algorithms to overcome the spurious vanishing problem. In particular, we integrate coefficient normalization with polynomial-based basis constructions, which do not need the proper ordering of monomials to process for basis constructions. We further propose a method that takes advantage of the iterative nature of basis construction so that computationally costly operations for coefficient normalization can be circumvented. Moreover, a coefficient truncation method is proposed for further accelerations. From the experiments, it can be shown that the proposed method overcomes the spurious vanishing problem, resulting in shorter feature vectors while sustaining comparable or even lower classification error.
\end{abstract}

\section{Introduction}
Discovering nonlinear structure behind data is a common task across various fields, such as machine learning, computer vision, and systems biology. An emerging concept from computer algebra for this task is the approximate vanishing ideal~\cite{heldt2009approximate,robbiano2010approximate}, which is defined as a set of polynomials that almost take a zero value, i.e., approximately vanish, for any point in data. Roughly, for a set of $n$-dimensional points $X\subset\mathbb{R}^n$,
\begin{align*}
\mathcal{I}_{\mathrm{app}}(X) & =\left\{ g\in\mathcal{P}_n\mid\forall\boldsymbol{x}\in X,g(\boldsymbol{x})\approx0\right\} ,
\end{align*}
where $\mathcal{P}_n$ is the set of all $n$-variate polynomials over the real numbers. 
An approximate vanishing polynomial $g\in\mathcal{I}_{\mathrm{app}}(X)$ holds $X$ as its approximate roots, which implies $g$ reflects the nonlinear structure underlying $X$. In particular, computing the basis set of approximate vanishing ideal has been attracting a lot of attention~\cite{heldt2009approximate,livni2013vanishing,limbeck2014computation,iraji2017principal}; such basis vanishing polynomials describe a system that has $X$ as approximate common roots, implying the nonlinear structure of data $X$ is captured in the system.
Various basis construction algorithms have been proposed and exploited in applications. For instance, nonlinear feature vectors of data are constructed for classifications
\cite{livni2013vanishing, shao2016nonlinear, hou2016discriminative};
independent signals are estimated for blind source separation tasks
\cite{kiraly2012regression, wang2018nonlinear}; nonlinear dynamical systems are reconstructed
from noisy observations~\cite{kera2016noise}; and so forth~\cite{torrente2009application,kera2016vanishing}. 

The essential ingredient for approximate vanishing ideal is the error tolerance $\epsilon$. 
A polynomial $g$ is approximately vanishing for a point $\boldsymbol{x}$ if $|g(\boldsymbol{x})|\le\epsilon$. An exact vanishing ideal, where $\epsilon=0$, can result in a corrupted model that overfits the noisy data and is far from the actual data structure. By setting proper $\epsilon > 0$, the basis set of approximate vanishing polynomials is expected to be a polynomial system that reflects the informative structure of noisy data.
However, this approximation gives rise to a new theoretical question: how can we properly evaluate the approximate vanishing of polynomials? Approximate vanishing polynomials change the extent of vanishing by simply rescaling. For example, a nonvanishing polynomial, $|g(\boldsymbol{x})|=2\epsilon\ge\epsilon$, can be easily converted into an approximate vanishing polynomial by rescaling its coefficients by $1/2$, i.e., $|(g/2)(\boldsymbol{x})|=\epsilon$. 
In other words, approximate vanishing can be achieved by small coefficients regardless of the roots of polynomials; such \emph{spurious} approximate vanishing polynomials do not hold any useful structure of data.
The converse is also true; polynomials that well describe data can be rejected as nonvanishing polynomials because of their large coefficients. Such polynomials are referred to as spurious nonvanishing polynomials.

In this paper, we address the aforementioned problem, spurious vanishing problem\footnote{Hereinafter, the spurious vanishing problem refers to both the problems on spurious vanishing polynomials and spurious nonvanishing polynomials.}, in basis construction algorithms of approximate vanishing ideal. In particular, we focus on polynomial-based algorithms, which are most commonly used basis construction algorithms of approximate vanishing ideals in applications other than computer algebra. To avoid the spurious vanishing problem, polynomials need to be normalized on some scale. We propose a general normalization scheme of polynomials that can work with various polynomial-based algorithms and normalization scales. We discuss the properties required by normalization scales and formulate our polynomial generation as a constraint optimization problem in the form of the generalized eigenvalue problem. The optimality, stability, and validity are guaranteed by rigorous theoretical analysis.
As a particular normalization, we consider coefficient normalization, which constraints the coefficient norms of a polynomial to be unity\footnote{The coefficient norm of a polynomial is the root square sum of the coefficients of monomials in the polynomial.}. This intuitive normalization has been considered in monomial-based basis construction algorithms~\cite{heldt2009approximate,limbeck2014computation,fassino2010almost}. However, introducing the coefficient normalization into the polynomial-based algorithms leads to significantly costly computation.
To sidestep costly polynomial expansions, we propose a method that obtains coefficients of polynomials by exploiting the iterative nature of basis construction and by precomputation. Furthermore, we propose a coefficient truncation method that enables the coefficient normalization to work much faster while giving up the exact calculation of coefficients.

In the experiments, we evaluate the Simple Basis Construction (SBC) algorithm that is designed as a simple polynomial-based algorithm for discussing and implementing our methods. Vanishing Component Analysis
(VCA;~\cite{livni2013vanishing}) is also adopted as a baseline method because it is the most widely used
basis construction algorithm of approximate vanishing ideal in various applications~\cite{zhao2014hand,yan18deep,wang2018nonlinear}.
Throughout the calculations, we show that VCA encounters severe coefficient growth
and decay, resulting in spurious vanishing polynomials with small coefficients and spurious nonvanishing polynomials with
large coefficients. When such polynomials are normalized to have a
unit coefficient norm, spurious vanishing polynomials turn into nonvanishing polynomials, and spurious nonvanishing polynomials turn into approximate vanishing polynomials. In contrast, the SBC algorithm with coefficient normalization does not encounter any spurious vanishing and nonvanishing polynomials. In classification tasks, the SBC algorithm with the coefficient normalization extracts shorter feature vectors while keeping comparable or even lower classification errors than VCA. 

Our contributions are summarized as follows:
\begin{itemize}
    \item We propose the first general method that can introduce normalization into polynomial-based basis construction algorithms of approximate vanishing ideal to avoid the spurious vanishing problem. 
    Rigorous theoretical analysis on the validity, optimality, and stability are provided.
    \item We propose two efficient methods for coefficient normalization, which is computationally costly to introduce into polynomial-based basis construction algorithms.
\end{itemize}

\section{Related Work}
\textbf{Monomial-based algorithms}---In computer algebra, basis polynomials of approximate vanishing ideals are generated from linear combinations of monomials during the basis construction~\cite{sauer2007approximate,abbott2008stable,heldt2009approximate,fassino2010almost,limbeck2014computation}. We refer to such basis construction algorithms as monomial-based algorithms. In these algorithms, the coefficient normalization is considered; distinct monomials are linearly combined with unit vectors, and thus, the norm of the coefficient vectors of the obtained polynomials is ensured to be unity. 
Most monomial-based algorithms are based on the Buchberger--M\"oller algorithm~\cite{moller1982construction} and its extension~\cite{kehrein2006computing}. The former computes the Gr\"obner bases of exact vanishing ideals and the latter computes the border bases, which are generalization of the Gr\"obner basis in the case of zero-dimensional ideals~\cite{setter2004numerical}. In these algorithms, one only has to handle a small number of monomials, typically, at most the number of input points.
Both the Gr\"obner bases and border bases hold powerful theoretical properties, which allow us to address various fundamental problems in computer algebra such as solving polynomials systems and the ideal membership problem. For approximate vanishing ideals, computed basis sets no longer demonstrate such properties, but are expected to show close behavior. In particular, border bases are considered more than the Gr\"obner bases for addressing approximate vanishing ideals because of its numerical stability in the coefficients and because of the numerical instability of the Gr\"obner bases~\cite{setter2004numerical,robbiano2010approximate,fassino2010almost}. A border basis (or a Gr\"obner basis) is defined with a monomial order to make the calculation well-defined. In computer algebra, the choice of the monomial order does not have much effect on solving most problems or a proper monomial order is known. However, when it comes to solving problems in other fields such as machine learning, the dependence of basis construction on the monomial order becomes problematic. There are a few monomial-based algorithms that work without monomial orders. 
For example, Sauer et al.~\cite{sauer2007approximate} proposed an algorithm to compute approximate H bases, which consist of homogeneous polynomials.
Hashemi et al.~\cite{hashemi2019computing} proposed a method to compute border bases for all possible monomial orders. In contrast to the algorithms based on the Buchberger--M\"oller algorithm, none of these algorithms work in polynomial time. Moreover, the computation technique used in the algorithm of Sauer et al. cannot be straightforwardly introduced into polynomial-based algorithms. The output of the algorithm of Hasemi et al. is a set of border bases. As shown in their experiments, the number of obtained polynomials across all the obtained number of border bases is large even for a small set of points.  

\noindent\textbf{Polynomial-based algorithms}---In contrast to monomial-based algorithms, polynomial-based algorithms do not rely on the monomial order and thus, are more commonly used in various fields. These algorithms consider linear combinations of polynomials instead of monomials during the basis construction. Although the combination vectors are unit, the terms of summed polynomials can cancel out or merge. As a consequence, the coefficients of polynomials decay and grow drastically, resulting in the spurious vanishing problem. Based on computer algebraic algorithms, Livni et al. proposed a pioneering polynomial-based algorithm, VCA, which is followed by various extensions. However, discarding monomial orders in basis construction leads to various theoretical issues that have not appeared in monomial-based algorithms. To our knowledge, these issues are rarely discussed in the literature. This paper focuses on the aforementioned issue of the coefficient decay and growth.

\section{Preliminaries\label{sec:Preliminaries}}
Hereinafter, the notations and definitions are based on~\cite{livni2013vanishing} and~\cite{kera2018approximate}, although font styles and descriptions are modified for ease of understanding and consistency.
\subsection{Definitions and Notations}
\begin{defn}[Vanishing Ideal]
Given a set of points $X\subset\mathbb{R}^n$, the vanishing ideal
of $X$ is the set of $n$-variate polynomials that take a zero value
(i.e., vanish) for any point in $X$. Formally, 
\begin{align*}
\mathcal{I}(X) & =\left\{ g\in\mathcal{P}_{n}\mid\forall\boldsymbol{x}\in X,g(\boldsymbol{x})=0\right\}.
\end{align*}
\end{defn}
\begin{defn}[Evaluation vector]
Given a set of points $X=\{\boldsymbol{x}_1,\boldsymbol{x}_2,...,\boldsymbol{x}_{|X|}\}\subset\mathbb{R}^n$, the evaluation vector of polynomial $h\in\mathcal{P}_n$ for $X$ is 
\begin{align*}
h(X) & =\begin{pmatrix}h(\boldsymbol{x}_{1}) & h(\boldsymbol{x}_{2}) & \cdots & h(\boldsymbol{x}_{|X|})\end{pmatrix}^{\top}\in\mathbb{R}^{|X|},
\end{align*}
where $|\cdot|$ denotes the cardinality of a set.
For a set of polynomials $H=\left\{ h_{1},h_{2},\ldots,h_{|H|}\right\}$,
its evaluation matrix for $X$ is $H(X)=(h_{1}(X)\ h_{2}(X)\ \cdots\ h_{|H|}(X))\in\mathbb{R}^{|X|\times |H|}$. 
\end{defn}

\begin{defn}[$\epsilon$-vanishing polynomial]
Given $\epsilon\ge 0$, a polynomial $g\in\mathcal{P}_n$ is an $\epsilon$-vanishing polynomial for a set of points $X\subset\mathcal{P}_n$ if $\|g(X)\|\le\epsilon$, where $\|\cdot\|$ denotes the Euclidean norm. Otherwise, $g$ is an $\epsilon$-nonvanishing
polynomial. 
\end{defn}

In the literature of vanishing ideals, polynomials are identified with their $|X|$-dimensional evaluation vectors.
Two polynomials $h$ and $\widetilde{h}$ are considered equivalent
if they have the same evaluation vector, i.e., $h(X)=\widetilde{h}(X)$.
When $h(X)=\boldsymbol{0}$, then $h$ is a vanishing polynomial for $X$.
It is worth noting that the sum of evaluation of polynomials equals to the evaluation of sum of the polynomials; that is, for a set of polynomials $H=\left\{ h_{1},...,h_{|H|}\right\} $
and weight vectors $\boldsymbol{v}=(v_{1,}v_{2},...,v_{|H|})^{\top}\in\mathbb{R}^{|H|}$,
\begin{align*}
H(X)\boldsymbol{v} & =(H\boldsymbol{v})(X),
\end{align*}
where $H\boldsymbol{v}=\sum_{i=1}^{|H|}v_{i}h_{i}$ defines the inner
product between a set $H$ and a vector $\boldsymbol{v}$. This special
inner product will be used hereafter. Similarly, for a matrix $V=\left(\boldsymbol{v}_{1}\cdots\boldsymbol{v}_{s}\right)\in\mathbb{R}^{|H|\times s}$, we define the product of $H$ and $V$ as $HV=\left\{ H\boldsymbol{v}_{1},...,H\boldsymbol{v}_{s}\right\}$.
In this way, polynomials and their sums are mapped to finite-dimensional
vectors, and the linear algebra can be used for the basis construction of vanishing ideals. 

\subsection{Simple Basis Construction Algorithm}
 Given a set of data points $X\subset\mathbb{R}^n$ and error tolerance $\epsilon$, the goal of the basis construction is to output $F,G\subset\mathcal{P}_n$, where $F$ is a basis set of $\epsilon$-nonvanishing polynomials and $G$ is a basis set of $\epsilon$-approximate vanishing polynomials of the approximate vanishing ideal of $X$. In the exact case ($\epsilon=0$), any vanishing polynomial $g\in\mathcal{I}(X)$ can be generated by $G$ as 
\begin{align}
    g = \sum_{g^{\prime}\in G} h_{g^{\prime}} g^{\prime}\label{eq:generator-g},
\end{align}
where $h_{g^{\prime}}\in \mathcal{P}_n$. This is similar to basis sets of linear subspaces except that the \textit{coefficients} are here polynomials. We define $\langle G \rangle = \{g\in\mathcal{P}_n\mid g= \sum_{g'\in G} h_{g'}g', h_{g'}\in\mathcal{P}_n\}$ as the set of polynomials generated by $G$.
Any polynomial $f\in\mathcal{P}_n$ can be described as \begin{align}
    f &= f^{\prime} + g^{\prime},\label{eq:generator-fg}
\end{align}
where $f^{\prime}\in\mathrm{span}(F)$ and $g^{\prime}\in\mathcal{I}(X)$. Here, $\mathrm{span}(F)$ denotes the set of all linear combinations of polynomials in $F$, i.e., $\mathrm{span}(F)=\{f\in\mathcal{P}_n\mid \sum_{f^{\prime}\in F}a_{f^{\prime}}f^{\prime}, a_{f^{\prime}}\in\mathbb{R}\}$.

There are many basis construction algorithms of approximate vanishing ideals.
Although our idea of normalization and its realization method work with most of them, to avoid unnecessarily abstract discussion, we focus on a simple polynomial-based basis construction algorithm, referred to as the SBC algorithm. 
The input to the SBC algorithm is a set of points $X\subset \mathbb{R}^{n}$ and error tolerance $\epsilon\ge 0$. The algorithm proceeds from degree-0 polynomials to higher degree polynomials. At each degree $t$, a set of nonvanishing polynomials $F_t$ and a set of vanishing polynomials $G_t$ are generated. We use notations $F^t=\bigcup_{\tau=0}^{t}F_{\tau}$ and $G^t=\bigcup_{\tau=0}^{t}G_{\tau}$. For $t=0$, $F_0 = \{m\}$ and $G_0 = \emptyset$, where $m$ is any nonzero constant polynomial. 
At each degree $t\ge 1$, the following procedures are conducted.

\paragraph*{Step 1: Generate a set of candidate polynomials $C_t$}
Pre-candidate polynomials of degree-$t$ for $t>1$ are generated by multiplying nonvanishing polynomials between $F_1$ and $F_{t-1}$.
\begin{align*}
    C_t^{\mathrm{pre}} = \{pq \mid p\in F_1, q\in F_{t-1}\}.
\end{align*}
At $t=1$, we use $C_1^{\mathrm{pre}}=\{x_1,x_2,...,x_n\}$, where $x_k$ are variables. A set of candidate polynomials is then generated via orthogonalization procedure. 
\begin{align}\label{eq:orthogonalization}
     C_{t} &= C_{t}^{\mathrm{pre}} - F^{t-1}F^{t-1}(X)^{\dagger}C_{t}^{\mathrm{pre}}(X),
 \end{align}
 where $\cdot^{\dagger}$ is a pseudo-inverse of a matrix.
\paragraph*{Step 2: Solve an eigenvalue problem for $C_t(X)$}{
Solve the following eigenvalue problem for the evaluation matrix $C_t(X)$,
\begin{align}\label{eq:evp}
    C_t(X)^{\top}C_t(X)V = V\Lambda,
\end{align}
where $V$ is a matrix that has eigenvectors $\boldsymbol{v}_1,\boldsymbol{v}_2,...,\boldsymbol{v}_{|C_t|}$ in its columns and $\Lambda$ is a diagonal matrix with eigenvalues $\lambda_1,\lambda_2,...,\lambda_{|C_t|}$ eigenvalues along its diagonal. 
}
\paragraph*{Step 3: Construct sets of basis polynomials}{
Basis polynomials are generated by linearly combining polynomials in $C_t$ with $\boldsymbol{v}_1,\boldsymbol{v}_2,...,\boldsymbol{v}_{|C_t|}$.
\begin{align*}
    F_t &= \{C_t\boldsymbol{v}_i\mid \sqrt{\lambda_i} > \epsilon, i=1,2,...,|C_t|\},\\
    G_t &= \{C_t\boldsymbol{v}_i\mid \sqrt{\lambda_i} \le \epsilon, i=1,2,...,|C_t|\}. 
\end{align*}
If $|F_t| = 0$, the algorithm terminates with output $F^t$ and $G^t$.
}

\begin{rem}
At Step~1, the orthogonalization procedure~(\ref{eq:orthogonalization}) makes the column space of $C_{t}(X)$ orthogonal to that of $F^{t-1}(X)$, aiming at focusing on the subspace of $\mathbb{R}^{|X|}$ that cannot be spanned by the evaluation vectors of polynomials of degree less than $t$ (note  that $C_t(X) = (I-F^{t-1}(X)F^{t-1}(X)^{\dagger})C_{t}^{\mathrm{pre}}(X)$, where $I$ is the identity matrix). 
\end{rem}

\begin{rem}\label{rem:extent-of-vanishing}
At Step~3, a polynomial $C_t\boldsymbol{v}_i$ is classified as an $\epsilon$-vanishing polynomial if $\sqrt{\lambda_i}\le \epsilon$ because $\sqrt{\lambda_i}$ equals the extent of vanishing of a polynomial $C_t\boldsymbol{v}_i$. In fact,
\begin{align*}
    \|(C_{t}\boldsymbol{v}_{i})(X)\|  =\sqrt{\boldsymbol{v}_{i}^{\top}C_{t}(X)^{\top}C_{t}(X)\boldsymbol{v}_{i}}=\sqrt{\lambda_{i}}.
\end{align*}
\end{rem}

The aforementioned algorithm shows the fundamental procedures of polynomial-based basis construction. This algorithm is quite similar to VCA~\cite{livni2013vanishing}, implying the elegance of VCA. Existing polynomial-based algorithms can be discussed based on the SBC algorithm by slightly changing each step and introducing additional procedures for various properties, such as stability, scalability, and compactness~\cite{livni2013vanishing,kiraly2014dual,kera2018approximate}. In this paper, we discuss using the SBC algorithm and will note when some algorithm-specific consideration is necessary. Henceforth, we refer to each step of the above algorithm as \texttt{Step1}, \texttt{Step2}, and \texttt{Step3}. 

\begin{thm}\label{thm:basis}\label{THM:BASIS}
When SBC runs with $\epsilon=0$ for a set of points $X$, the output basis sets $G$ and $F$ satisfy the following. \begin{itemize}
    \item Any vanishing polynomial $g\in\mathcal{I}(X)$ can be generated by $G$, i.e., $g\in\langle G\rangle$. 
    \item Any polynomial $f$ can be represented by $f = f^{\prime} + g^{\prime}$, where $f^{\prime}\in\mathrm{span}(F)$ and $g^{\prime}\in\langle G\rangle$.
    \item For any $t$, any degree-$t$ vanishing polynomial $g\in\mathcal{I}(X)$ can be generated by $G^t$, i.e., $g\in\langle G^t\rangle$. 
    \item For any $t$, any degree-$t$ polynomial $f$ can be represented by $f = f^{\prime} + g^{\prime}$, where $f^{\prime}\in\mathrm{span}(F^t)$ and $g^{\prime}\in\langle G^t\rangle$.
\end{itemize}
\end{thm}

\begin{proof} 
The SBC algorithm is identical to VCA up to a constant factor in basis polynomials. Specifically, VCA set $F_0=\{m\}=\{1/\sqrt{|X|}\}$, and normalize polynomials $f\in F_t$ by $\|f(X)\|$ at each degree. 
Thus, from Theorem~5.2 in~\cite{livni2013vanishing}, which shows that VCA satisfies Theorem~\ref{thm:basis}, we can conclude that SBC also satisfies Theorem~\ref{thm:basis}.
\end{proof}
In section~\ref{sec:basis-construction}, the SBC algorithm is redefined by extending \texttt{Step2} to \texttt{Step2}$^{\prime}$ with an introduction of normalization. We will prove that Theorem~\ref{thm:basis} still holds after the redefinition (cf. Theorem~\ref{thm:basis-is-basis}). The pseudocodes for the redefined SBC will be provided in Algorithms~\ref{alg:SBC},~\ref{alg:nm}, and~\ref{alg:bc}.

\subsection{Toy example}
First, we examine the computation of SBC with simple data points to provide an intuition for the spurious vanishing problem. Let us consider SBC for $X= \{(1+\xi,1),(1,1+\xi),(-1+\xi,-1+\xi),(-1,-1)\}$, where $\xi=0.1$. We use $\epsilon=\sqrt{4\xi^2}=2\xi=0.2$. The constant polynomial $m$ at degree 0 is set to $m=1$, and thus, $F_0=\{f_1:=1\},G_0=\{\}$. We use $x$ and $y$ for variables and the values are rounded for illustration while keeping the use of ``$=$'' notation. The coefficient vector of a  polynomial $h$ is denoted by $\mathfrak{n}_{\mathrm{c}}(h)$, e.g., $\mathfrak{n}_{\mathrm{c}}(1-x+2xy) = (1,-1,0,0,2,0)^{\top}$, and $\|\mathfrak{n}_{\mathrm{c}}(h)\|$ is called the coefficient norm of $h$.

\paragraph*{Degree $t=1$}
At \texttt{Step 1}, the candidate polynomials are generated as follows.
\begin{align*}
    C_1 
    &= \{x, y\} - \{1\}
    \frac{1}{4}\left(
    \begin{array}{cccc}
        1 & 1 & 1 & 1
    \end{array}
    \right)
    \left(\begin{array}{cc}
       1.1  & 1 \\
        1 & 1.1 \\
        -0.9 & -0.9 \\
        -1 & -1
    \end{array}\right),\\
    &= \{x - 0.05, y-0.05\}.
\end{align*}
The eigenvalues of $C_1(X)$ is $\{0.01, 8.01\}$, and the corresponding eigenvectors are $\boldsymbol{v}_1=(-0.707, 0.707)^{\top}$ and $\boldsymbol{v}_2=(0.707, 0.707)^{\top}$. According to the square roots of eigenvalues $\{\sqrt{0.01}, \sqrt{8.01}\} = \{0.1, 2.83\}$, 
\begin{align*}
    G_1 &= \{g_1:=C_1\boldsymbol{v}_1 = -0.707(x-0.05)+0.707(y-0.05)\}, \\
    F_1 &= \{f_2:=C_1\boldsymbol{v}_2 = 0.707(x-0.05)+0.707(y-0.05)\}.
\end{align*}

\paragraph*{Degree $t=2$}
The set of precandidate polynomials is $C_2^{\mathrm{pre}}=\{f_2^2\}= \{(0.707(x-0.05)+0.707(y-0.05))^2\}$, and thus, the set of candidate polynomials of degree 2 is 
\begin{align*}
    C_2 
    &= \{(0.707(x-0.05)+0.707(y-0.05))^2\} \\
    &\quad\ - \{1, f_2\}
    \left(
    \begin{array}{cc}
        \boldsymbol{1}_4 & f_2(X)
    \end{array}
    \right)^{\dagger}
    f_2^2(X),\\
    &= \{0.5x^2 + xy + 0.5y^2 - 0.096x - 0.0963y - 2.00\},
\end{align*}
where $\boldsymbol{1}_4\in\mathbb{R}^4$ is the all-one vector.
The eigenvalue and eigenvector of $C_2(X)$ are $0.78$ and $(1.0)$, respectively. Thus, $G_2=\{\}$ and $F_2=\{f_3:=0.5x^2 + xy + 0.5y^2 - 0.096x - 0.0963y - 2.00\}$.

\paragraph*{Degree $t=3$}
The set of precandidate polynomials is $C_3^{\mathrm{pre}}=\{f_2f_3\}$ and thus, the set of candidate polynomials of degree 3 is 
\begin{align*}
    C_3 
    &= \{f_2f_3\} \\
    &\quad\ - \{1, f_2, f_3\}
    \left(
    \begin{array}{ccc}
        \boldsymbol{1}_4 & f_2(X) & f_3(X)
    \end{array}
    \right)^{\dagger}
    (f_2f_3)(X),\\
    &= \{0.354x^3 + 1.061x^2y + 1.061xy^2 + 0.354y^3 \\
    &\quad\quad + 0.601x^2 + 1.202xy + 0.601y^2- 1.549x\\ 
    &\quad\quad  - 1.549y - 2.671\}.
\end{align*}
The eigenvalue and eigenvector of $C_3(X)$ are $0$ and $(1.0)$, respectively. Thus, $G_3=\{g_2\}:=C_3$ and $F_2=\{\}$, and the algorithm terminates with the outputs $G=\bigcup_{t=0}^{3}G_t = \{g_1,g_2\}$ and $F=\bigcup_{t=0}^{3}F_t=\{1, f_2,f_3\}$.

\begin{figure}
    \centering
    \includegraphics[width=\linewidth]{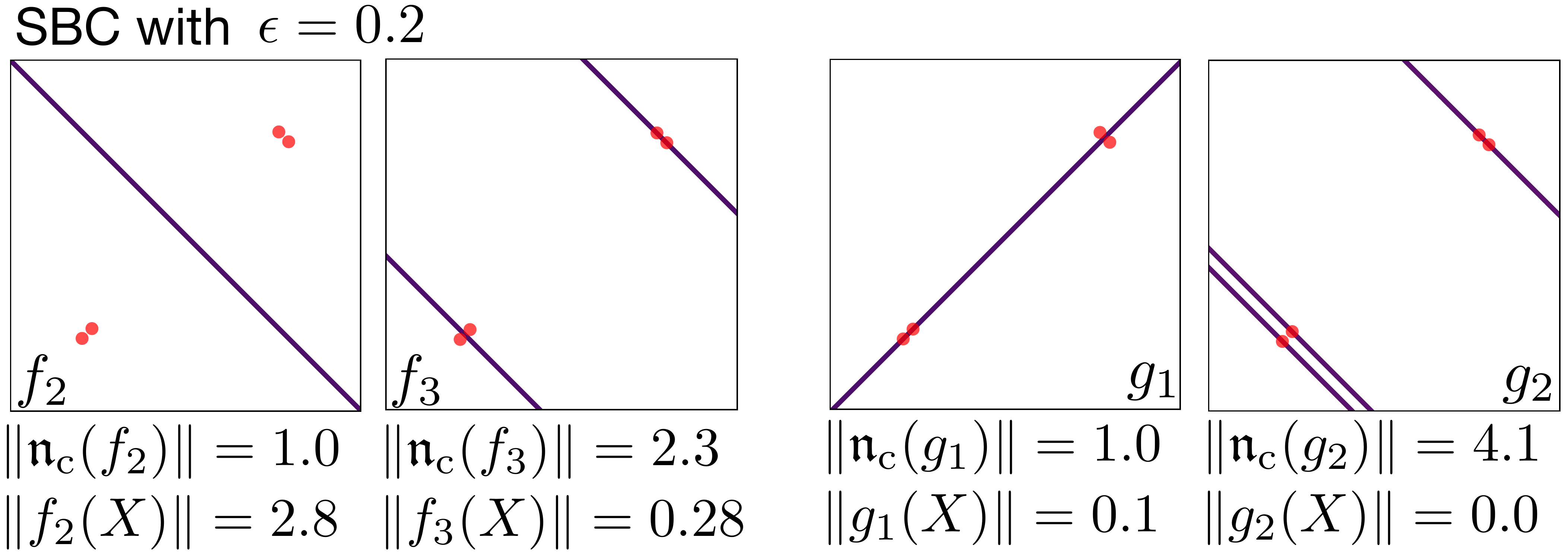}
    \caption{$\epsilon$-nonvanishing polynomials ($f_2$ and $f_3$) and $\epsilon$-vanishing polynomials ($g_1$ and $g_2$) obtained by SBC for $X= \{(1.1,1),(1,1.1),(-0.9,-0.9),(-1,-1)\}$ with $\epsilon=0.2$. The coefficient vector of a polynomial $h$ is denoted by $\mathfrak{n}_{\mathrm{c}}(h)$. Here, $f_3$ is classified as a nonvanishing polynomial based on its extent of vanishing $\|f_3(X)\|=0.28>\epsilon$, which is overrated because of the relatively large coefficient norm $\|\mathfrak{n}_{\mathrm{c}}(f_3)\|=2.3$.}
    \label{fig:example}
\end{figure}
\begin{figure}
    \centering
    \includegraphics[width=\linewidth]{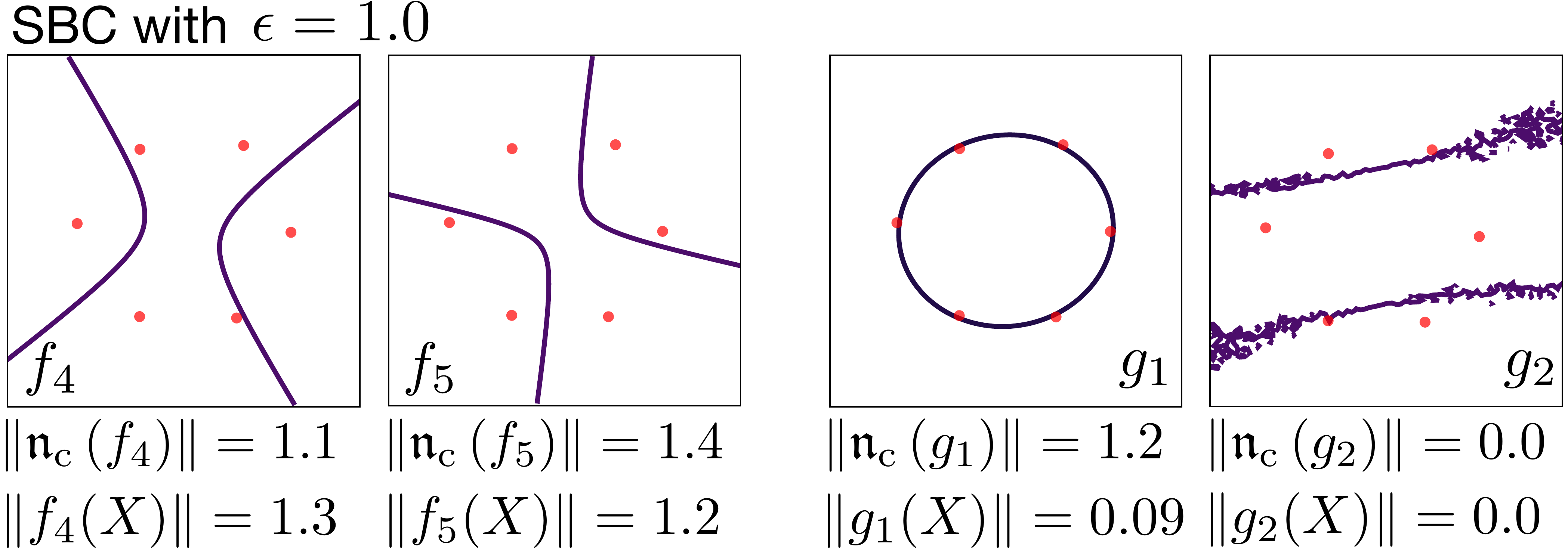}
    \caption{Some of the $\epsilon$-nonvanishing polynomials ($f_4$ and $f_5$) and $\epsilon$-vanishing polynomials ($g_1$ and $g_2$) obtained by SBC for $X$ (perturbed $X_0 = \{(\cos(k\pi/3),\sin(k\pi/3))\}_{k=0,1,...,5}$) with $\epsilon=1.0$. 
    Here, $g_2$ is classified as an $\epsilon$-vanishing polynomials based on its extent of vanishing $\|g_2(X)\|= 0.0\le\epsilon$, which is underrated because of its small coefficient norm $\|\mathfrak{n}_{\mathrm{c}}(g_2)\|=0.0$ (only approximately equal to zero but rounded off as 0.0).}
    \label{fig:example2}
\end{figure}
We provide contour plots of the obtained polynomials (except $f_1$) in Fig.~\ref{fig:example}. The coefficient norm and the extent of the vanishing of these polynomials are listed below the plots. Here, $f_3$ is classified as a nonvanishing polynomial based on $\|f_3(X)\|=0.28>\epsilon$, although the lines approximately pass through the points. This results from the large coefficient norm of $f_3$, i.e., $\|\mathfrak{n}_{\mathrm{c}}(f_3)\|=2.3$. In other words, $f_3$ is a spurious nonvanishing polynomial. In fact, once $f_3$ is normalized to $f^{\prime}_3:=f_3/\|\mathfrak{n}_{\mathrm{c}}(f_3)\|$, it becomes an $\epsilon$-vanishing polynomial because of $\|f^{\prime}_3(X)\|\approx 0.12 \le \epsilon$. We observe the opposite case when we apply SBC to another set of points $X=\{(0.53,0.87),(-0.49,0.83),(-1.1,0.1),(-0.5,-0.81),\\ (-0.46, -0.83),(0.99,0.02)\}$ with $\epsilon=1.0$. This $X$ is generated by perturbing points of $X_0 = \{(\cos(k\pi/3),\sin(k\pi/3))\}_{k=0,1,...,5}$ on a unit circle with the zero-mean additive Gaussian noise (the standard deviation is 0.01). As shown in Fig.~\ref{fig:example2}, $g_2$ is classified as an $\epsilon$-vanishing polynomial, although the lines on the plot do not pass through the points at all. This is because the magnitude of the coefficients of $g_2$ is extremely small. In other words, $g_2$ is a spurious vanishing polynomial.

\section{Proposed Method}
A polynomial can be approximately vanishing only because of its small coefficients---this is the spurious vanishing problem. To avoid this problem, we propose that approximate vanishing polynomials (and nonvanishing polynomials) be normalized by some scale, such that the spurious vanishing polynomials are properly rescaled and their actual behavior for input points becomes evident. 

Here, we describe the proposed methods that deal with the challenges of introducing normalization. We intend to answer the following questions: how do we optimally generate basis polynomials under a normalization, such as coefficient normalization?~(Section~\ref{sec:basis-construction}); how do we efficiently extract coefficients from polynomials and manipulate them for coefficient normalization~(Section~\ref{sec:coefficient-normalization})?

\subsection{Polynomial-based basis construction with Normalization}\label{sec:basis-construction}
Here, we describe a proposed method, which enables SBC algorithm to construct nonvanishing and vanishing polynomials under given normalization. This method is general enough to be applied to other polynomial-based basis construction algorithms, accompanied with the following advantages: (i) it requires to rewrite only a few lines of original algorithms\footnote{Except the calculation of values that are necessary for normalization, which depends on which normalization is used.}; (ii) it is not limited to coefficient normalization; (iii) it can inherit most properties of the original algorithms. For simplicity, we first focus on coefficient normalization and then provide the general description. 

The coefficient vector $\mathfrak{n}_\mathrm{c}(h)$ of a polynomial $h$ is defined as a vector that lists the coefficients of monomials, e.g., $\mathfrak{n}_\mathrm{c}(1-x+2x^3)=(1,-1,0,2)^{\top}$. The order of listing is arbitrary if it is consistent across polynomials. The length of coefficient vectors also has an arbitrarity; for example, $\mathfrak{n}_\mathrm{c}(1-x+2x^3)=(1,-1,0,2,0,0)^{\top}$ is also valid (the last two zeros correspond to $x^4$ and $x^5$). When we perform an operation (e.g., the dot product) on two coefficient vectors, the shorter vector is extended by padding zeros. 

In the coefficient normalization, we evaluate the extent of vanishing of $g$ for $X$ by normalizing $g$ with respect to its coefficient norm as
\begin{align*}
    \frac{g}{\|\mathfrak{n}_\mathrm{c}(g)\|}.
\end{align*}
Because spurious vanishing polynomials have small coefficient norms, these polynomials are largely scaled in the normalization as previously d. Similarly, spurious nonvanishing polynomials, which are polynomials that are nonvanishing because of their unreasonably large coefficients, are rescaled to have a moderate scale of coefficients.

As previously mentioned, the coefficient normalization has been considered in monomial-based algorithms, but not in polynomial-based algorithms, leading polynomial-based algorithms to suffer from the spurious vanishing problem. One reason that polynomial-based algorithms fail to consider the coefficient normalization is that it has been unknown how to optimally generate combination vectors ($\boldsymbol{v}_i$ of \texttt{Step2}) under this normalization. 
Recall that in \texttt{Step3}, a new polynomial $g$ is generated by linearly combining candidate polynomials in $C_t = \{c_1,c_2,...,c_{|C_t|}\}$. This can be formulated as
\begin{align*}
    g = \sum_{i=1}^{|C_t|} v_i c_i = C_t\boldsymbol{v},
\end{align*}
where $\boldsymbol{v}=(v_1,v_2,...,v_{|C_t|})^{\top}$ is a combination vector to be sought. 
The coefficient vector of $g$ is 
\begin{align*}
    \mathfrak{n}_{\mathrm{c}}(g) = \sum_{i=1}^{|C_t|} v_i \mathfrak{n}_{\mathrm{c}}(c_i) = \mathfrak{n}_{\mathrm{c}}(C_t)\boldsymbol{v},
\end{align*}
where we use a slight abuse of notation such that $\mathfrak{n}_{\mathrm{c}}(C_t)$ is a matrix whose $i$-th column is $\mathfrak{n}_{\mathrm{c}}(c_i)$.
Now, suppose that we want to find a polynomial that achieves the tightest vanishing under the coefficient normalization, which is formulated as follows:
\begin{align*}
    \min_{\boldsymbol{v}}\ \|C_t(X)\boldsymbol{v}\|^2, \quad\mathrm{s.t.}\  \|\mathfrak{n}_{\mathrm{c}}(C_t)\boldsymbol{v}\|^2 = 1.
\end{align*}
This type of minimization problem is well-known to be reduced to a generalized eigenvalue problem,
\begin{align}\label{eq:gep-coefcase}
    C_t(X)^{\top}C_t(X)\boldsymbol{v}_{\mathrm{min}} = \lambda_{\mathrm{min}} \mathfrak{n}_{\mathrm{c}}(C_t)^{\top}\mathfrak{n}_{\mathrm{c}}(C_t)\boldsymbol{v}_{\mathrm{min}},
\end{align}
where $\lambda_{\mathrm{min}}$ is the smallest generalized eigenvalue, and $\boldsymbol{v}_{\mathrm{min}}$ is the corresponding generalized eigenvector. We later show that the generalized eigenvectors of the $r$-smallest generalized eigenvalues generate polynomials that minimize the sum of the extent of vanishing under the normalization.
Therefore, to introduce the coefficient normalization, we only need to replace the eigenvalue problem~(\ref{eq:evp}) in \texttt{Step2} with the generalized eigenvalue problem~(\ref{eq:gep-coefcase}).

We now provide a general description of our method. 
\begin{defn}[Normalization mapping] Let \label{def:normalization-operator}
$\mathfrak{n}:\mathcal{P}_n\to \mathbb{R}^{\ell}$ be a mapping that satisfies the following.
\begin{itemize}
    \item $\mathfrak{n}$ is a linear mapping, i.e., $\mathfrak{n}(ah_1+bh_2) = a\mathfrak{n}(h_1)+b\mathfrak{n}(h_2)$, for any $a,b\in\mathbb{R}$ and any polynomials $h_1,h_2\in\mathcal{P}_n$.
    \item The dot product is defined between normalization components; that is, $\langle\mathfrak{n}(h_1),\mathfrak{n}(h_2) \rangle$ is defined for any polynomials $h_1,h_2\in\mathcal{P}_n$.
    \item $\mathfrak{n}(h)$ takes a zero value if and only if $h$ is the zero polynomial.
\end{itemize}
Then, $\mathfrak{n}$ is a normalization mapping.
$\mathfrak{n}(h)$ is called the normalization component of $h$, and $\|\mathfrak{n}(h)\|$ is called the norm (or $\mathfrak{n}$-norm) of $h$.
\end{defn}
Note that from the first requirement, the norm of the  zero polynomial needs to be zero value. The third requirement insists that the converse is also true, and this is the case for the coefficient normalization; that is, if $\mathfrak{n}_{\mathrm{c}}(h)=\boldsymbol{0}$, then $h$ is the zero polynomial. Let us consider $|C_t|$-dimensional vectors $\boldsymbol{v}_1=(v_1^{(1)},v_1^{(2)},...,v_1^{(|C_t|)})^{\top}$ and $\boldsymbol{v}_2=(v_2^{(1)},v_2^{(2)},...,v_2^{(|C_t|)})^{\top}$.
Using the first and second properties of $\mathfrak{n}$,
\begin{align*}
    \langle \mathfrak{n}(C_t\boldsymbol{v}_1), \mathfrak{n}(C_t\boldsymbol{v}_2)\rangle
    &= \Bigl\langle \mathfrak{n}\Bigl(\sum_{i}c_i v_1^{(i)}\Bigr), \mathfrak{n}\Bigl(\sum_{j}c_jv_2^{(j)}\Bigr)\Bigr\rangle,\\
    &= \sum_{i,j}\langle \mathfrak{n}(c_i), \mathfrak{n}(c_j)\rangle v_{1}^{(i)}v_{2}^{(j)},\\
    & = \boldsymbol{v}_1^{\top} \mathfrak{N}(C_t)\boldsymbol{v}_2,
\end{align*}
where $\mathfrak{N}$ is a mapping that gives a matrix $\mathfrak{N}(C_t)\in\mathbb{R}^{|C_t|\times |C_t|}$ for $C_t$, the $(i,j)$-th entry of which is $\langle \mathfrak{n}(c_i), \mathfrak{n}(c_j)\rangle$. 
With the constraints $\langle \mathfrak{n}(C_t\boldsymbol{v}_k), \mathfrak{n}(C_t\boldsymbol{v}_l)\rangle=\delta_{kl}$ for every $k$ and $l$, where $\delta_{kl}$  is the Kronecker delta, basis polynomials are generated by solving the following generalized eigenvalue problem.
\begin{align}\label{eq:gep}
    C_t(X)^{\top}C_t(X)V = \mathfrak{N}(C_t)V\Lambda,
\end{align}
where $\Lambda$ is a diagonal matrix containing generalized eigenvalues $\lambda_1,...,\lambda_{|C_t|}$, and $V$ is a matrix whose $i$-th column is the generalized eigenvector $\boldsymbol{v}_i$ corresponding to $\lambda_i$.
To summarize, \texttt{Step2} is replaced with the following \texttt{Step2$^{\prime}$} to introduce a normalization. The pseudocodes of SBC with \texttt{Step2}$^{\prime}$ are provided in Algorithms~\ref{alg:SBC},~\ref{alg:nm},~and~\ref{alg:bc}.
\paragraph*{Step 2$^{\prime}$: Solve the generalized eigenvalue problem for $C_t(X)$}
Solve the generalized eigenvalue problem~(\ref{eq:gep}) to obtain the generalized eigenvectors $\boldsymbol{v}_1,\boldsymbol{v}_2,...,\boldsymbol{v}_{|C_t|}$ and the generalized eigenvalues $\lambda_1,\lambda_2, ...,\lambda_{|C_t|}$.
\begin{rem}
In addition to replacing \texttt{Step2} with \texttt{Step2}$^{\prime}$, we set $F_0 = \{m\} = \{1\}$ for consistency in the coefficient normalization. 
\end{rem}

\begin{algorithm}[t]
\caption{Simple Basis Construction}\label{alg:SBC}

\begin{algorithmic}[1]
\Require{ $X\subset\mathbb{R}^n, \epsilon\ge 0, m\ne 0, \mathfrak{n}:\mathcal{P}_n\to\mathbb{R}^{\ell}$} 
\Ensure{$G,F$} 
\State{$G_0, F_0 = \{\}, \{m\}$} 
\State{$G, F = G_0, F_0$} 
\For{$t=1,2,...$ } 
    \If{$t\le 1$} \Comment{\texttt{Step 1}}
        \State{$C_t = \{x_1,x_2,...,x_n\}$}
    \Else
        \State{$C_t = \{pq \mid p\in F_1, q\in F_{t-1}\}$}
    \EndIf
    \State{$C_t = C_t - FF(X)^{\dagger}C_t(X)$}
    \State{$\mathfrak{N}(C_t) = \texttt{NormlizationMatrix}(C_t, \mathfrak{n})$}
    \State{$G_t, F_t = \texttt{BasisConstruction}(C_t, \mathfrak{N}(C_t), X, \epsilon)$} 
    \State{$G, F = G \cup G_t, F \cup F_t$} 
    \If{$F_t = \emptyset$} 
        \State{\textbf{terminate}}
    \EndIf
\EndFor
\end{algorithmic} 
\end{algorithm}

\begin{algorithm}[t]\caption{\texttt{NormalizationMatrix}}\label{alg:nm}

\begin{algorithmic}[1]
\Require{ $C_t, \mathfrak{n}$ }  \Comment{$C_t = \{c_1,c_2,...,c_{|C_t|}\}$}
\Ensure{$\mathfrak{N}(C_t)$}
    \State{$\mathfrak{n}(C_t) = \left(\begin{array}{cccc}
        \mathfrak{n}(c_1) & \mathfrak{n}(c_2) & \cdots & \mathfrak{n}(c_{|C_t|})
    \end{array}\right)
    $}
    \State{$\mathfrak{N}(C_t) = \mathfrak{n}(C_t)^{\top}\mathfrak{n}(C_t)$}
    \Statex\Comment{For the unnormalized case, simply return $\mathfrak{N}(C_t) = I$.}
\end{algorithmic} 
\end{algorithm}

\begin{algorithm}[t]
\caption{\texttt{BasisConstruction}}\label{alg:bc}

\begin{algorithmic}[1] 
\Require{ $C_t, \mathfrak{N}(C_t), X, \epsilon$ } 
\Ensure{$G,F$} 
    \State{$C_t(X)^{\top}C_t(X)V = \mathfrak{N}(C_t)V\Lambda$} \Comment{\texttt{Step 2}$^{\prime}$}
    \Statex \Comment{$\boldsymbol{v}_i$: the $i^{\text{th}}$ column of $V$, $\lambda_i$: the $i^{\text{th}}$ diagonal entry of $\Lambda$.}
    \State{$G = \{C_t\boldsymbol{v}_i \mid \sqrt{\lambda_i}\ge \epsilon, i=1,2,...,|C_t|\}$} \Comment{\texttt{Step 3}}
    \State{$F = \{C_t\boldsymbol{v}_i \mid \sqrt{\lambda_i}> \epsilon, i=1,2,...,|C_t|\}$}
\end{algorithmic} 
\end{algorithm}

The following theorem supports the validity of this replacement of \texttt{Step2} with \texttt{Step2}$^{\prime}$.
\begin{thm}\label{thm:basis-is-basis}\label{THM:BASIS-IS-BASIS}
Given a normalization mapping $\mathfrak{n}$, the SBC algorithm with \texttt{Step2}$^{\prime}$ satisfies Theorem~\ref{thm:basis}.
\end{thm}
\begin{proof}
We prove the claim by induction with respect to degree $t$. 
Let us denote by $F_t$ and $G_t$ the basis sets obtained at degree-$t$ iteration in SBC with \texttt{Step2}. 
For the corresponding items in SBC with \texttt{Step2}$^{\prime}$, we put a tilde on the symbols such as $\widetilde{F}_t$ and $\widetilde{G}_t$. From Theorem~\ref{thm:basis}, we know that collecting $F_t$ and $G_t$ gives complete basis sets for both nonvanishing and vanishing polynomials. Here, we prove the claim by comparing $(\widetilde{F}_t, \widetilde{G}_t)$ with $(F_t, G_t)$. Specifically, we show $\mathrm{span}(F_t)=\mathrm{span}(\widetilde{F}_t)$ and $\langle G^t\rangle = \langle \widetilde{G}^t\rangle$. 
Note that $\mathrm{span}(F_t)\supset\mathrm{span}(\widetilde{F}_t)$ and $\langle G^t\rangle \supset \langle \widetilde{G}^t\rangle$ are evident because $\widetilde{F}_t$ and $\widetilde{G}_t$ are generated by assigning additional constraints on normalization in the original generation of $F_t$ and $G_t$.
Thus, the main goal is to prove the reverse inclusions 
$\mathrm{span}(F_t)\subset\mathrm{span}(\widetilde{F}_t)$ and $\langle G^t\rangle \subset \langle \widetilde{G}^t\rangle$.

At $t=1$, it is evident that $\mathrm{span}(F_1)=\mathrm{span}(\widetilde{F}_1)$ and $\langle G^1\rangle=\langle\widetilde{G}^1\rangle$. For $t\le\tau$, we assume $\mathrm{span}(F_t) = \mathrm{span}(\widetilde{F}_t)$ and $\langle G^t\rangle = \langle \widetilde{G}^t\rangle$. Then we can show $\mathrm{span}(C_{\tau+1}^{\mathrm{pre}})=\mathrm{span}(\widetilde{C}_{\tau+1}^{\mathrm{pre}})$ and $\mathrm{span}(C_{\tau+1})=\mathrm{span}(\widetilde{C}_{\tau+1})$. In fact, it is $pq\in \mathrm{span}(\widetilde{C}_{\tau+1}^{\mathrm{pre}})$ for any $pq\in C_{\tau+1}^{\mathrm{pre}}$, where $p\in F_1$ and $q\in F_{\tau}$, because $p\in\mathrm{span}(\widetilde{F}_1)$ and $q\in\mathrm{span}(\widetilde{F}_{\tau})$, and vice versa. The orthogonalization projects $\mathrm{span}(C_{\tau+1}^{\mathrm{pre}})$ to subspace $\mathrm{span}(C_{\tau+1})$, which are orthogonal to $\mathrm{span}(F^{\tau})$ in terms of the evaluation, i.e., $\mathrm{span}(C_{\tau+1}(X)) \perp \mathrm{span}(F^{\tau}(X))$. From $\mathrm{span}(C_{\tau+1}^{\mathrm{pre}})=\mathrm{span}(\widetilde{C}_{\tau+1}^{\mathrm{pre}})$ and $\mathrm{span}(F^{\tau}(X))=\mathrm{span}(\widetilde{F}^{\tau}(X))$, the orthogonalization projects $C_{\tau+1}^{\mathrm{pre}}$ and $\widetilde{C}_{\tau+1}^{\mathrm{pre}}$ into the same subspace, i.e., $\mathrm{span}(C_{\tau+1})=\mathrm{span}(\widetilde{C}_{\tau+1})$.

Next, we show $\mathrm{span}(F_{\tau+1})=\mathrm{span}(\widetilde{F}_{\tau+1})$ by showing $\mathrm{span}(F_{\tau+1})\subset \mathrm{span}(\widetilde{F}_{\tau+1})$.
 Let us consider a nonzero polynomial $f\in F_{\tau+1}$. As shown previously, $f\in\mathrm{span}(C_{\tau+1}) = \mathrm{span}(\widetilde{C}_{\tau+1})$. By construction, $f\in\mathrm{span}(\widetilde{F}_{\tau+1})$ unless $\|\mathfrak{n}(f)\|=0$. On the other hand, if $\|\mathfrak{n}(f)\|=0$, then $f=0$ because of the third requirement of $\mathfrak{n}$. 
 Therefore, $\mathrm{span}(F_{\tau+1})=\mathrm{span}(\widetilde{F}_{\tau+1})$.

We can show $\langle G^{\tau+1}\rangle = \langle \widetilde{G}^{\tau+1}\rangle$ in a similar manner. Let us consider a vanishing polynomial $g\in G_{\tau+1}$. Then $g\in \mathrm{span}(\widetilde{G}_{\tau+1})$ unless $\|\mathfrak{n}(g)\| = 0$, which implies $g=0$. Therefore, $\langle G^{\tau+1}\rangle = \langle \widetilde{G}^{\tau+1}\rangle$.
\end{proof}

An intuitive explanation of Theoreom~\ref{thm:basis-is-basis} is as follows. Let us consider two processes of SBC, one with \texttt{Step2} and the other with \texttt{Step2}$^{\prime}$. If a symbol $A$ is used for the former process, we put a tilde on it as $\widetilde{A}$ for the latter process. Now, at degree $t$, SBC with \texttt{Step2}$^{\prime}$ finds basis sets $\widetilde{F}_t$ and $\widetilde{G}_t$. As a consequence, the basis sets $\widetilde{F}_t$ and $\widetilde{G}^t$ cover parts of the space that is covered by the basis sets $F_t$ and $G^t$ obtained in the process of SBC with \texttt{Step2}. This is because \texttt{Step2}$^{\prime}$ solves the generalized eigenvalue problem~(\ref{eq:gep}), where additional constraints regarding normalization are imposed on the eigenvalue problem~(\ref{eq:evp}) that  is solved in \texttt{Step2}. The key claim is that basis polynomials dropped from $F_t$ and $G_t$ are redundant basis polynomials.
The redundant basis polynomials (let one of them be $h$) are those with their norm to be zero (i.e., $\|\mathfrak{n}(h)\|=0$). From the third property of Definition~\ref{def:normalization-operator}, $\|\mathfrak{n}(h)\|=0$ implies that $h$ is the zero  polynomial; thus, $h$ need not be included in $F_t$ nor $G_t$.
Note that existing methods do not exclude even such a zero polynomial from basis sets. 
The following theorem shows the optimality of \texttt{Step2}$^{\prime}$. 
\begin{thm}\label{thm:gep}\label{THM:GEP}
Let $r$ be an integer such that $1\le r\le\mathrm{rank}(\mathfrak{N}(C_t))$. The $r$ generalized eigenvectors $\boldsymbol{v}_1,\boldsymbol{v}_2,...,\boldsymbol{v}_r$ of~(\ref{eq:gep}), which correspond to the $r$-smallest generalized eigenvalues,
generate polynomials $C_t\boldsymbol{v}_1,C_t\boldsymbol{v}_2,...,C_t\boldsymbol{v}_r$, whose square sum of the extent of vanishing achieves the minimum under the orthonormal constraint on the normalization components of polynomials.
\end{thm}
\begin{proof}
From Remark~\ref{rem:extent-of-vanishing} and the discussion of deriving~(\ref{eq:gep}), it can be readily shown that the following problem needs to be solved. 
\begin{align*}
\min_{V\in\mathbb{R}^{|C_t|\times r}}\  & \mathrm{Tr}\left(V^{\top}C_{t}(X)^{\top}C_{t}(X)V\right),\\
\mathrm{s.t.}\  & V^{\top}\mathfrak{N}(C_t)V=I.
\end{align*}
It is known that the column vectors of the optimal $V$ of the above problem are generalized eigenvectors corresponding to the $r$-smallest generalized eigenvalues of $(C_t(X)^{\top}C_t(X), \mathfrak{N}(C_t))$. The following proof is based on~\cite{vidal2016gpca}.

Introducing a Lagrange multiplier $\widetilde{\Lambda}\in\mathbb{R}^{|C_{t}|\times|C_{t}|}$, we have
\begin{align*}
\mathcal{L} & =\frac{1}{2}\mathrm{Tr}\left[V^{\top}C_{t}(X)^{\top}C_{t}(X)V\right] +\frac{1}{2}\mathrm{Tr}\left[(I-V^{\top}\mathfrak{N}(C_t)V)\widetilde{\Lambda}\right].
\end{align*}
Note that here $\widetilde{\Lambda}$ is symmetric due to the symmetric
constraint, but not diagonal in general. By differentiating $\mathcal{L}$
with $V$ and setting it to zero, 
\begin{align*}
\left.\frac{\partial L}{\partial V}\right|_{V=\widetilde{V}} & =C_{t}(X)^{\top}C_{t}(X)\widetilde{V}-\mathfrak{N}(C_t)\widetilde{V}\widetilde{\Lambda}=0.
\end{align*}
Thus, we obtain
\begin{align*}
C_{t}(X)^{\top}C_{t}(X)\widetilde{V} & =\mathfrak{N}(C_t)\widetilde{V}\widetilde{\Lambda}.
\end{align*}
Note that this is not yet a generalized eigenvalue problem because
$\widetilde{\Lambda}$ is not diagonal. Because $\widetilde{\Lambda}$ is symmetric,
it has an eigenvalue decomposition $\widetilde{\Lambda}=R\Lambda R^{\top}$
for some orthonormal matrix $R$ and diagonal matrix $\Lambda$. Thus,
we obtain the generalized eigenvalue problem~(\ref{eq:gep}) by $V=\widetilde{V}R$.
The cost function is 
\begin{align*}
\mathrm{Tr}\left[V^{\top}C_{t}(X)^{\top}C_{t}(X)V\right] & =\mathrm{Tr}\left[V\mathfrak{N}(C_t)V\Lambda\right]=\mathrm{Tr}\left[\Lambda\right],
\end{align*}
which means that the $r$ smallest generalized eigenvalues minimize the cost.
\end{proof}

It is known that in practice, we need to solve the following problem instead of~(\ref{eq:gep}) for numerical stability. 
\begin{align}
C_{t}(X)^{\top}C_{t}(X)V & =(\mathfrak{N}(C_t)+\alpha I)V\Lambda,\label{eq:gep-stable}
\end{align}
where $\alpha$ is a small positive constant. Such $\alpha$ is typically set to a small multiple of the average eigenvalue of $\mathfrak{N}(C_t)$, i.e., $\mathrm{Tr}(\mathfrak{N}(C_t))/|C_{t}|$~\cite{friedman1989regularized}. 
It can be shown that such an addition by $\alpha I$ only gives a slight change both on the extent of vanishing and the normalization of obtained polynomials.  

\begin{thm}
Let $\{\boldsymbol{v}_{1}^{\alpha} , \boldsymbol{v}_{2}^{\alpha},\ldots , \boldsymbol{v}_{|C_t|}^{\alpha}\}$ be the generalized eigenvectors of~(\ref{eq:gep-stable}) for
$\alpha \ge 0$. Both the extent of vanishing and the norm of $C_{t}\boldsymbol{v}_{k}^{\alpha}$ differ only by $O(\alpha)$
from those of $C_{t}\boldsymbol{v}_{i}^0$. Specifically,
\begin{align*}
    \lambda_k^{0}-\lambda_k^{\alpha} &= \frac{\lambda_k^{0}}{1+\alpha\|\boldsymbol{v}_k^0\|^2},\\
    -\alpha\|\boldsymbol{v}_k^0\|^2\frac{\lambda_k^{0}}{\lambda_{\mathrm{min}}^0}  + O(\alpha^2) 
    &\le \sqrt{(\boldsymbol{v}_k^{0})^{\top}\mathfrak{N}(C_t)\boldsymbol{v}_k^{0}}\\
    &\quad - \sqrt{(\boldsymbol{v}_k^{\alpha})^{\top}\mathfrak{N}(C_t)\boldsymbol{v}_k^{\alpha}}, \\
    &\le -\alpha\|\boldsymbol{v}_k^0\|^2\frac{\lambda_k^{0}}{\lambda_{\mathrm{max}}^0}  + O(\alpha^2), 
\end{align*}
where $\lambda_{\mathrm{min}}^0$ and $\lambda_{\mathrm{max}}^0$ are the smallest and the largest eigenvalues of $\lambda_{k}^{0}$ for $k=1,...,|C_t|$, respectively.
\label{thm:stable-gep}
\label{THM:STABLE-GEP}
\end{thm}
The proof of Theorem~\ref{thm:stable-gep} relies on the following lemma. 
\begin{lem}\label{lem:perturbation}
Suppose square matrices $A,B\in\mathbb{R}^{n\times n}$ are symmetric and positive-semidefinite, and $\mathrm{nullspace}(A)\supset \mathrm{nullspace}(B)$. 
Let us consider a perturbed generalized eigenvalue problem $A\boldsymbol{v}_k^{\alpha} = \lambda_k^{\alpha}(B+\alpha I)\boldsymbol{v}_k^{\alpha}$ for a small nonnegative constant $\alpha$, ($k=1,...,\mathrm{rank}(B)$). 
Then
\begin{align*}
    \lambda_k^{0}-\lambda_k^{\alpha} &= \frac{\lambda_k^{0}}{1+\alpha\|\boldsymbol{v}_k^0\|^2},\\
    -\alpha\|\boldsymbol{v}_k^0\|^2\frac{\lambda_k^{0}}{\lambda_{\mathrm{min}}^{0}}  + O(\alpha^2) 
    &\le \|W\boldsymbol{v}_k^{0}\| - \|W\boldsymbol{v}_k^{\alpha}\| \\ 
    &\le -\alpha\lambda_k^{0}\|\boldsymbol{v}_k^0\|^2\frac{\lambda_k^{0}}{\lambda_{\mathrm{max}}^{0}}  + O(\alpha^2), 
\end{align*}
where $\lambda_{\mathrm{min}}^0$ and $\lambda_{\mathrm{max}}^0$ are the smallest and the largest generalized eigenvalue among $\lambda_1^{0}, ..., \lambda_{\mathrm{rank}(B)}^{0}$, respectively.
\end{lem}

\begin{proof}
A symmetric and positive-semidefinite matrix $B$ has orthonormal eigenvectors $\boldsymbol{u}_1, ..., \boldsymbol{u}_n$,  where the first $\mathrm{rank}(B)$ eigenvectors span the column space of $B$ and the rest span the nullspace. Note that the generalized eigenvectors $\boldsymbol{v}_k^{0},(k=1,...,\mathrm{rank}(B))$ are mutually linearly independent due to $(\boldsymbol{v}_k^{0})^{\top}B\boldsymbol{v}_l^{0} = \delta_{kl}$. Hence, $\{\boldsymbol{v}_1^{0}, ..., \boldsymbol{v}^{0}_{\mathrm{rank}(B)}, \boldsymbol{u}_{\mathrm{rank}(B)+1}, ..., \boldsymbol{u}_n\}$ becomes a complete basis of $\mathbb{R}^n$. Therefore, for any $k$,
\begin{align*}
    \boldsymbol{v}_k^{\alpha} &= \sum_{i=1}^{\mathrm{rank}(B)} a_i^{\alpha}\boldsymbol{v}_i^{0} + \sum_{i=\mathrm{rank}(B)+1}^n a_i^{\alpha}\boldsymbol{u}_i,
\end{align*}
for some $a_1^{\alpha}, ..., a_n^{\alpha}$.
Substituting the above expression into  $A\boldsymbol{v}_k^{\alpha}=\lambda_k^{\alpha}(B+\alpha I)\boldsymbol{v}_k^{\alpha}$, we have
\begin{align}\label{eq:gep-expansion}
    A\sum_{i=1}^{\mathrm{rank}(B)} a_i^{\alpha}\boldsymbol{v}_i^{0} &= \lambda_k^{\alpha}B\sum_{i=1}^{\mathrm{rank}(B)} a_i^{\alpha}\boldsymbol{v}_i^{0} \nonumber \\
    &+\alpha \lambda_k^{\alpha}\left(\sum_{i=1}^{\mathrm{rank}(B)} a_i^{\alpha}\boldsymbol{v}_i^{0} + \sum_{i=\mathrm{rank}(B)+1}^n a_i^{\alpha}\boldsymbol{u}_i\right), 
\end{align}
where we used $A\boldsymbol{u}_{i}=0$ for $i > \mathrm{rank}(B)$ because $\mathrm{nullspace}(A)\supset \mathrm{nullspace}(B)$. 
Multiplying both sides by $(\boldsymbol{v}_k^0)^{\top}$ from the left, 
\begin{align*}
    a_k^{\alpha}\lambda_k^{0} &= \lambda_k^{\alpha}a_k^{\alpha} + \alpha \lambda_k^{\alpha}a_k^{\alpha}\|\boldsymbol{v}_k^0\|^2,
\end{align*}
where $(\boldsymbol{v}_k^0)^{\top}A\boldsymbol{v}_i^{0} = \lambda_k^0(\boldsymbol{v}_k^0)^{\top}B\boldsymbol{v}_i^0 = \lambda_k\delta_{ik}$ is used.
When $\alpha$ is sufficiently small, $a_k^{\alpha}$ is nonzero (almost 1). Thus, dividing by $a_k^{\alpha}$, we obtain
\begin{align}\label{eq:stability-of-lambda}
    \lambda_k^{\alpha} &= \frac{\lambda_k^{0}}{1+\alpha\|\boldsymbol{v}_k^0\|^2} = \lambda_k^{0} + O(\alpha) > 0.
\end{align}
Next, by simple calculations, 
\begin{align*}
    \lambda_k^{\alpha} &= (\boldsymbol{v}_k^{\alpha})^{\top}A\boldsymbol{v}_k^{\alpha},\\
    & = \left(\sum_{i=1}^{\mathrm{rank}(B)} a_i^{\alpha} \boldsymbol{v}_i^{0}\right)^{\top}A\sum_{i=1}^{\mathrm{rank}(B)} a_i^{\alpha}\boldsymbol{v}_i^{0}, \\
    &= \sum_{i=1}^{\mathrm{rank}(B)} (a_i^{\alpha})^2\lambda_i^{0}, \\
    &= \lambda_k^{0} + \sum_{i=1}^{\mathrm{rank}(B)} \left((a_i^{\alpha})^2 - \delta_{ik}\right)\lambda_i^{0}.
\end{align*}
Note that $\lambda_k^{\alpha} = \lambda_k^{0} -\alpha\lambda_k^{0}\|\boldsymbol{v}_k^0\|^2 + O(\alpha^2)$ from~(\ref{eq:stability-of-lambda}).
Therefore, 
\begin{align*}
    \sum_{i=1}^{\mathrm{rank}(B)}\left((a_i^{\alpha})^2 - \delta_{ik}\right)\lambda_i^{0} &= -\alpha\lambda_k^{0}\|\boldsymbol{v}_k^0\|^2 + O(\alpha^2).
\end{align*}
For sufficiently small $\alpha$, the right-hand side is negative because $\alpha\lambda_k^0\|\boldsymbol{v}_k^0\|^2 > 0$ and $O(\alpha^2)\approx 0$, and thus, the left-hand side is also negative. Hence, 
\begin{align*}
    -\alpha\|\boldsymbol{v}_k^0\|^2\frac{\lambda_k^{0}}{\lambda_{\mathrm{min}}^{0}}  + O(\alpha^2) &\le \sum_{i=1}^{\mathrm{rank}(B)} ((a_i^{\alpha})^2 - \delta_{ik}) \\
    &\le -\alpha\lambda_k^{0}\|\boldsymbol{v}_k^0\|^2\frac{\lambda_k^{0}}{\lambda_{\mathrm{max}}^{0}}  + O(\alpha^2).
\end{align*}
\end{proof}

\begin{proof}[Proof of Theorem~\ref{thm:stable-gep}]
To simplify the notations, let $A=C_{t}(X)^{\top}C_{t}(X)$ and $B=\mathfrak{N}(C_t)$. Let us consider 
\begin{align}
A\boldsymbol{v}_{k}^{\alpha} & =\lambda_{k}^{\alpha}(B+\alpha I)\boldsymbol{v}_{k}^{\alpha},\label{eq:perturbed-gep}
\end{align}
where $\alpha I$ is a small perturbation on $B$ and $\lambda_{k}^{\alpha}$
is the perturbed $k$-th generalized eigenvalue. We cannot directly apply the standard matrix perturbation theory, which assumes positive-definite $B$ and describes $\boldsymbol{v}_{k}^{\alpha}$ by a linear combination of unperturbed generalized eigenvectors. 
In our case, $B$ is positive-semidefinite, and thus there are only $\mathrm{rank}(B)$ generalized eigenvectors. Hence, the generalized eigenvectors do not form a complete basis of $\mathbb{R}^{|C_{t}|}$, and $\boldsymbol{v}_{k}^{\alpha}$ cannot always be described by these generalized eigenvectors. 
Fortunately, the theorem above holds using the fact $\mathrm{nullspace}(A)\supset\mathrm{nullspace}(B)$, where $\mathrm{nullspace(\cdot)}$ denotes the nullspace of a given matrix. This relation holds because any vector $\boldsymbol{v}\in\mathrm{nullspace}(B)$ implies the zero polynomial according to the third requirement for $\mathfrak{n}$. 
From Lemma~\ref{lem:perturbation}, we conclude that the claim holds.
\end{proof}

\subsection{Coefficient normalization}\label{sec:coefficient-normalization}
Introducing the coefficient normalization into the basis construction needs large computational cost. There are two sources that cause the cost:
(i) we need to expand polynomials to obtain their coefficient vectors because in our case, polynomials are in the nested sum-product form of polynomials due to the repetition of \texttt{Step1} and \texttt{Step3} along the degree. In general, such an expansion is computationally expensive because one has to manipulate exponentially many monomials from the expansion in the worst case. (ii) Even after the polynomial expansion, the obtained coefficient vectors of polynomials are exponentially long in general. Specifically, a degree-$t$ $n$-variate polynomial has a coefficient vector of length $\binom{n+t}{n}$. 
Here, we propose two methods for each of the two challenges above.

\subsubsection{Circumventing polynomial expansion}
The main idea for circumventing polynomial expansion is to hold coefficient vectors of polynomials separately and update these vectors by applying to them the equivalent transformations that are applied to the corresponding polynomials. For example, let us consider a weighted sum $af+bg$ of two polynomials $f$ and $g$ by weights $a,b\in\mathbb{R}$. Then, the coefficient vector of $af+bg$ is also a weighted sum $\mathfrak{n}_{\mathrm{c}}(af+bg) =  a\mathfrak{n}_{\mathrm{c}}(f)+b\mathfrak{n}_{\mathrm{c}}(g)$. In contrast to the weighted sum case, it is not easy to calculate the coefficient vector of the product of polynomials, e.g., $\mathfrak{n}_{\mathrm{c}}(fg)$. We encounter such a case at \texttt{Step1}, where the precandidate polynomials are generated from the multiplication across linear polynomials and nonlinear polynomials. We will now deal with this problem.

Let us consider $n$-variate polynomials. Let $M_{n}^{t}=\binom{n+t-1}{t}$ and $M_{n}^{\le t}=\binom{n+t}{t}$ be the number of $n$-variate monomials of degree $t$ and of degree up to $t$,
respectively. For simple description, we assume that monomials and coefficients are indexed in the degree-lexicographic order. For instance, in the two-variate case, the degree-lexicographic order is $1,x,y,x^{2},xy,y^{2},x^{3},...,$ and so forth ($x,y$ are variables). We will refer to ``the $i$-th monomial'' according to this ordering. Now, we consider a matrix that extends a coefficient vector of a degree-$t$ polynomial to that of a degree-$(t+1)$ polynomial after multiplication by a linear polynomial.

\begin{rem}
Given a linear polynomial $p$, there is a matrix $R_{p}^{\le t}\in\mathbb{R}^{M_{n}^{\le t+1}\times M_{n}^{\le t}}$ such that
\begin{align*}
    \mathfrak{n}_{\mathrm{c}}(pq) = R_{p}^{\le t}\mathfrak{n}_{\mathrm{c}}(q),
\end{align*} 
for any polynomial $q$ of degree $t$.
\end{rem}
The existence of such matrix $R_{p}^{\le t}$ will soon become evident (see~\cite{vidal2005generalized} for the case of homogeneous polynomials). Suppose a linear polynomial $p$ is described by $p=\sum_{i=0}^n b_kx_k$, where $b_k\in\mathbb{R}$ are coefficients and $x_{1},...,x_{n}$ are variables. For convenience, we use a notation $x_{0}=1$. Then, $R_{p}^{\le t}$ can be described as 
\begin{align*}
R_{p}^{\le t} & =\sum_{k=0}^{n}b_{k}R_{x_{k}}^{\le t},
\end{align*}
because as observed above, the coefficient vector of the weighted sum of polynomials is the weighted sum of their coefficient vectors.
Now, the existence of $R_{x_{k}}^{\le t}$ is evident. In fact, the $(i,j)$-th entry of $R_{x_{k}}^{\le t}$ takes value one if the $i$-th monomial becomes the $j$-th monomial by the multiplication with $x_k$, and otherwise the $(i,j)$-th entry of $R_{x_{k}}^{\le t}$ is zero value. Note that $R_{x_{k}}^{\le t}$ is not dependent on input data (except the number of variables), and thus, we can compute these matrices in advance. 
Different monomials are mapped to different monomials after multiplied by $x_k$. Thus, each column of $R_{x_{k}}^{\le t}$ has exactly one nonzero entry (and it is 1), implying $R_{x_{k}}^{\le t}$ is a sparse matrix with only $M_{n}^{\le t}$ entries, which can be efficiently handled. Moreover, we can represent $R_{x_k}^{\le t}$ in a block diagonal matrix for $1\le k \le n$,  
\begin{align*}
R_{x_{k}}^{\le t} & =\left(\begin{array}{cc}
R_{x_{k}}^{\le t-1} & O\\
O & R_{x_{i}}^{t}
\end{array}\right),
\end{align*}
where $O$ is the zero matrix, and $R_{x_{k}}^{t}\in\mathbb{R}^{M_{n}^{t+1}\times M_{n}^{t}}$
is a submatrix of $R_{x_{k}}^{\le t}$ that corresponds to the mapping from degree-$t$ monomials to degree-$(t+1)$ monomials. For $k=0$, $R_{x_0}^{\le t}\in\mathbb{R}^{M_{n}^{\le t+1}\times M_{n}^{\le t}}$ is a rectangular diagonal matrix with value one along its diagonal.

In summary, in the basis construction, we first hold the coefficient vectors of $F_1$ besides polynomials (or their evaluation vectors). Then, for each $p\in F_1$, we linearly combine precomputed $R_{x_k}^{\le 1}$ to obtain $R_{p}^{\le 1}$. Using these matrices, we can obtain the coefficient vectors of $C_2^{\mathrm{pre}}$. We then extend $R_{p}^{\le 1}$ to $R_{p}^{\le 2}$ by appending $R_{p}^2$, which is a linear combination of the precomputed $R_{x_k}^{2}$ to obtain $C_3^{\mathrm{pre}}$. In this way, we can directly manipulate coefficient vectors without performing costly polynomial expansions.

In addition to its less computational cost, we have another practical advantage in our approach that skips the polynomial expansion: it can work with the fast numerical implementation of basis construction. In the numerical implementation, a polynomial is expressed by its evaluation vector instead of a symbolic entity~(for example, see the code of~\cite{livni2013vanishing} provided in the first author's web page). Because we only know an evaluation vector of a polynomial in the numerical implementation, the ``polynomial'' cannot be expanded because its symbolic form is unknown. 
Numerical implementations work much faster because in practice, symbolic operations are much slower than the same number of numerical operations (matrix--vector operations). Also, it is slow to evaluate symbolic entities, although many evaluations are necessary to obtain evaluation vectors of polynomials.

\subsubsection{Coefficient truncation for acceleration}\label{sec:coefficient-truncation}
We here describe the coefficient truncation method to deal with significantly long coefficient vectors. We propose to truncate coefficient vectors based on the importance of the corresponding monomials. In particular, at each degree $t$, we only keep degree-$t$ monomials that have large coefficients in the degree-$t$ nonvanishing polynomials $F_{t}=\left\{ f_{1},f_{2},...,f_{s}\right\} $. Although this strategy is simple, our coefficient truncation method has an interesting contrast to a monomial-based algorithm as will be further discussed.

The specific procedures of the proposed coefficient truncation are as follows. Let $\mathfrak{n}_{\mathrm{c}}^t: \mathcal{P}_n\to M_{n}^t$ be a mapping that gives the coefficient vector corresponding to degree-$t$ monomials of the given polynomial; thus, $\mathfrak{n}_{\mathrm{c}}^t(f_i)$ is a subvector of $\mathfrak{n}_{\mathrm{c}}(f_i)$. With the same abuse of notation of $\mathfrak{n}_{\mathrm{c}}(F_t)$, we define $\mathfrak{n}_{\mathrm{c}}^t(F_t)$ as a matrix whose $i$-th column is $\mathfrak{n}_{\mathrm{c}}^t(f_i)$.
Note that the $j$-th row of $\mathfrak{n}_{\mathrm{c}}^t(F_t)$ corresponds to the coefficients of the $j$-th degree-$t$ monomial across polynomials of $F_t$. Let $\Delta_j$ be the norm of the $j$-th row of $\mathfrak{n}_{\mathrm{c}}^t(F)$. Then, setting a threshold parameter $\theta$, we select monomials individually from larger $\Delta_j$ as long as the following holds, 
\begin{align}
    \sum_{j\in\mathcal{B}_{t}}\Delta_{j}^2 \le\theta^{2}\label{eq:truncation-threshold},
\end{align}
where $\mathcal{B}_{t}$ is the index set of selected degree-$t$ monomials. We also truncate $R_{x_k}^{t}$ of the previous section to size $M_{n}^{t+1}\times |\mathcal{B}_t|$, which becomes a sparse matrix with $|\mathcal{B}_t|$ nonzero entries. Because the coefficient norm is always underestimated from the truncated coefficient vectors, we need to scale the norm of the truncated coefficient vectors according to the truncation rate. To this end, we calculate a rate $\gamma_t$, which is a ratio of the root square sum of the preserved coefficients to the root square sum of the full coefficients; that is, 
\begin{align*}
 \gamma_t &= \sqrt{\sum_{j\in\mathcal{B}_t}\Delta_j^2\left/\|\mathfrak{n}_{\mathrm{c}}^t(F)\|^2\right.}.   
\end{align*}
The product $\prod_{\tau=1}^t \gamma_{\tau}$ approximates the truncation rate up to degree $t$. The normalization matrix for \texttt{Step2}$^{\prime}$ is set to
\begin{align}\label{eq:truncated-normalization-matrix}
    \widetilde{\mathfrak{n}}_{\mathrm{c}}(C_t)^{\top}\widetilde{\mathfrak{n}}_{\mathrm{c}}(C_t)\left(\prod_{\tau=1}^t \gamma_{\tau}\right)^{-1},
\end{align}
where $\widetilde{\mathfrak{n}}_\mathrm{c}(C_t)$ is a matrix whose column vectors are consist of the truncated coefficient vectors of $C_t$.

The proposed coefficient truncation is similar to a monomial-based algorithm,  approximate Buchberger--M\"oller algorithm (ABM algorithm;~\cite{limbeck2014computation}). This algorithm proceeds from lower to higher degree monomials, while updating a set of ``important'' monomials $\mathcal{O}$ (called an order ideal), which corresponds to the basis set $F$ of nonvanishing polynomials of the SBC algorithm. Given a new monomial $b$, if the evaluation vector of $b$ cannot be well approximated by a linear combination of monomials in $\mathcal{O}$, the ABM algorithm assorts $b$ into $\mathcal{O}$. More specifically, if $b(X)\approx \sum_{m\in\mathcal{O}}c_{m}m(X)$ for some coefficients $\{c_{m}\}_{m\in\mathcal{O}}$, 
then $b-\sum_{m\in\mathcal{O}}c_{m}m$
is an approximate vanishing polynomial and $b$ is discarded; otherwise
$b$ is appended to $\mathcal{O}$. Importantly, monomials divisible by $b$ (i.e., multiples of $b$) need not be
considered at a higher degree, which reduces the number of monomials
to handle. It is shown that $|\mathcal{O}|\le|X|$; thus, the number
of monomials to handle does not explode. 

\begin{table}
\caption{Summary of Example~\ref{exa:two-monoms}}\label{table:example1}
\centering
\begin{tabular}{|c|c|c|}
\hline 
Setting & \multicolumn{2}{c|}{\makecell{$m_{1}\prec m_{2}$,\\ $m_{1}(X)=km_{2}(X)$}}\tabularnewline
\hline 
\hline 
 & \multirow{1}{*}{Nonvanishing basis} & Most important monomial\tabularnewline
\hline 
ABM & $m_{1}$ & $m_{1}$\tabularnewline
\hline 
SBC-$\mathfrak{n}_{\mathrm{c}}$ & $\frac{k}{\sqrt{1+k^{2}}}m_{1}+\frac{1}{\sqrt{1+k^{2}}}m_{2}$ & $\begin{cases}
m_{1} & (k<1)\\
m_{2} & (k>1)
\end{cases}$\tabularnewline
\hline 
\end{tabular}
\end{table}

\begin{table}
\caption{Summary of Example~\ref{exa:three-monoms}}\label{table:example2}\centering
\begin{tabular}{|c|c|c|}
\hline 
Setting & \multicolumn{2}{c|}{\makecell{$m_{1}\prec m_{2}\prec m_{3}$,\\ $m_{1}(X)=km_{2}(X)$,\\ $m_{i}(X)^{\top}m_{3}(X)=0,(i=1,2)$}}\tabularnewline
\hline 
\hline 
 & \multirow{1}{*}{Nonvanishing basis} & Top-2 important monomials\tabularnewline
\hline 
ABM & $m_{1},m_{3}$ & $m_{1},m_{3}$\tabularnewline
\hline 
SBC-$\mathfrak{n}_{\mathrm{c}}$ & $\frac{k}{\sqrt{1+k^{2}}}m_{1}+\frac{1}{\sqrt{1+k^{2}}}m_{2},m_{3}$ & $m_{3},\begin{cases}
m_{1} & (k<1)\\
m_{2} & (k>1)
\end{cases}$\tabularnewline
\hline 
\end{tabular}
\end{table}

The proposed coefficient truncation is distinct from the strategy of ABM algorithm in that it is fully data driven, whereas ABM algorithm relies on a specific monomial ordering. Now, we provide two examples to highlight the difference in their strategies (also see Tables~\ref{table:example1} and~\ref{table:example2}).
\begin{example}
Let us consider to decide the more ``important'' monomial from $m_1$ and $m_2$ of the same degree, where $m_{1}(X)=km_{2}(X)\ne\boldsymbol{0}$ for a constant $k$ and $m_{1}\prec m_{2}$ for some monomial order.
The ABM strategy selects $m_{1}$ as the more important monomial because of $m_{1}\prec m_{2}$, whereas the proposed strategy selects $m_1$ when $k> 1$ and $m_2$ when $k< 1$.
\label{exa:two-monoms}
\end{example}
The process of our strategy to select more important monomial from $m_1$ and $m_2$ is as follows. By solving a generalized eigenvalue problem, the following nonvanishing polynomial is obtained. 
\begin{equation}
\frac{k}{\sqrt{1+k^{2}}}m_{1}+\frac{1}{\sqrt{1+k^{2}}}m_{2}.\label{eq:shared-dir-example}
\end{equation}
According to the coefficients of each monomial, $m_{1}$ is considered more important when $k > 1$, and $m_2$ is considered more important when $k < 1$.
This result is quite natural because, for example, $k> 1$ implies that $m_1$ is more nonvanishing than $m_2$ because $\|m_1(X)\| = k\|m_2(X)\|$.
In this way, our strategy of keeping monomials with larger coefficients is fully data-driven.
In contrast, due to the predefined monomial order $m_{1}\prec m_{2}$, ABM strategy consistently selects
$m_{1}$, regardless of $k$. 

Next, we introduce an additional monomial $m_{3}$.
\begin{example}
In addition to $m_1$ and $m_2$ in Example~\ref{exa:two-monoms}, let us consider a monomial $m_{3}$, where the degree of $m_3$ is the same as that of $m_1$ and $m_2$, the  evaluation
vector $m_{3}(X)$ that is orthogonal to $m_{1}(X)$ and $m_{2}(X)$, and $m_1\prec m_2 \prec m_3$. The ABM strategy selects $m_1$ and $m_3$ as the most important and the next most important monomial, respectively, due to the monomial order and the relation of evaluation vectors. On the other hand, the proposed strategy considers $m_3$ as the most important and $m_1$ or $m_2$ as the next most important based on $k$.
 \label{exa:three-monoms}
\end{example}
The process of our strategy to first select $m_3$ and then $m_1$ is as follows. 
By solving a generalized eigenvalue problem, we obtain $m_3$ and~(\ref{eq:shared-dir-example}) as nonvanishing polynomials. Based on the coefficients of the monomials, $m_3$ is considered the most important by our strategy. From the previous discussion, the second most important monomial is $m_1$ if $k>1$, otherwise $m_2$. We emphasize that the magnitude of the coefficient of $m_{3}$ is larger than that of $m_{1}$
and $m_{2}$ because the coefficient norm of nonvanishing polynomials are normalized and
$m_{3}$ completely takes the coefficient norm
1, whereas $m_{1}$ and $m_{2}$ share the coefficient norm
1 by $k/\sqrt{1+k^{2}}$ and $1/\sqrt{1+k^{2}}$ due to their mutually
linearly dependent evaluation vectors.
This result implies that our strategy gives a priority to monomials that
have unique evaluation vectors, such as $m_{3}$.  The monomials that have
similar evaluation vectors to others, such as $m_{1}$ and $m_{2}$, tend to have moderate coefficients. Hence, our coefficient truncation method takes monomials from those with unique evaluation vectors to those with less unique evaluation vectors. 

As a consequence of truncating coefficient vectors, coefficient matrix $\mathfrak{n}_{\mathrm{c}}(C_t)^{\top}\mathfrak{n}_{\mathrm{c}}(C_t)$ in~(\ref{eq:gep-coefcase}) is replaced with the one computed from the truncated coefficient vectors~(\ref{eq:truncated-normalization-matrix}). The coefficient norm of the obtained polynomials is no longer equal but only close to unity. However, we can still calculate the exact evaluation of polynomials by keeping their evaluation vectors and coefficient vectors separate. Thus, the generalized eigenvalues $\{\lambda_{i}\}_{i=1,2,...,|C_t|}$ at \texttt{Step2}$^{\prime}$ maintain the exact value of the square extent of vanishing. 

It is difficult to estimate the error caused by the coefficient truncation because basis construction proceeds iteratively, and error gets accumulated.
The following theorem gives a theoretical lower bound of the
coefficient truncation without any loss. 
\begin{thm}\label{thm:truncation-limit}\label{THM:TRUNCATION-LIMIT} For an exact calculation of coefficients, we need at least $|F_t|$
monomials for each degree $t$. 
\end{thm}
\begin{proof}
Evaluation vectors of nonvanishing polynomials are mutually orthogonal. They form a basis that spans the subspace of $\mathbb{R}^{|X|}$
for a set of data points $X$, and appending new nonvanishing polynomials gradually
completes the basis. 
By construction, $\mathrm{span}(F_t(X))$ is a subspace of $\mathrm{rank}(|F_t|)$ that is orthogonal to $\mathrm{span}(F^{t-1}(X))$. A degree-$t$ polynomial is the sum of a linear combination of degree-$t$ monomials and a polynomial of degree less than $t$. Thus, we need $|F_t|$ or more degree-$t$ monomials to obtain $F_t$ whose evaluation vectors are mutually orthogonal and orthogonal to $\mathrm{span}(F^{t-1}(X))$, which concludes that the claim holds true. 
\end{proof}
Theorem~\ref{thm:truncation-limit} states the minimal number of monomials required to perform an exact calculation of coefficient vectors. The equality holds when the evaluation vectors of monomials are always orthogonal until the termination. Since this is too optimistic in practice, we propose to keep $O(|F_t|)$ coefficients at each degree $t$. Suppose the basis construction terminates at $t=T$. Then, the length of coefficient vectors at $T-1$ is $O(|F^{T-1}|)$. The matrix used to calculate the coefficient vectors of $C_T$ is $O(|F^{T-1}|)$-sparse. The number of new monomials in this step is $O(|F^{T-1}|n)$. It is known $|F^T|\le |X|$ because the evaluation vectors of $F^T$ (approximately) spans the $\mathbb{R}^{|X|}$. Therefore, the coefficient truncation yields coefficient vectors in polynomial-order length $O(n|X|)$. As a consequence, computing $\mathfrak{n}_{\mathrm{c}}(C_t)^{\top}\mathfrak{n}_{\mathrm{c}}(C_t)$ in~(\ref{eq:gep}) costs $O(|C_t|^2\cdot n|X|)=O(n^3|X|^3)$. This is acceptable when one considers the cost of solving~(\ref{eq:gep}) is also $O(|C_t|^2)=O(n^3|X|^3)$.

Lastly, we consider another idea for approximating coefficient norm of polynomials: how about calculating the norm of the evaluation vector of a polynomial at randomly sampled points to infer the coefficient norm? Unfortunately, this strategy does not work as further discussed. 
Let $\mathcal{V}_n^t(\cdot)$ be the Veronese map, which gives the evaluations of $n$-variate monomials of degree up to $t$. For instance, $\mathcal{V}_2^2(\boldsymbol{x})=\left(1, x_1, x_2, x_1^2, x_1x_2, x_2^2\right)\in\mathbb{R}^{1\times 6}$. For a set of points $X$, we define $\mathcal{V}_n^t(X)\in\mathbb{R}^{|X|\times  M_{n}^{\le t}}$ as a matrix whose $i$-th row is the Veronese map of the $i$-th point. Now, let us consider a polynomial $g=C_t\boldsymbol{v}$ and its evaluation for randomly sampled points $Y$.
\begin{align*}
    \|g(Y)\|^2 &=\|C_t(Y)\boldsymbol{v}\|^2, \\
    &= \|\mathcal{V}_n^t(Y)\mathfrak{n}_{\mathrm{c}}(C_t)\boldsymbol{v}\|^2, \\ 
    &= \boldsymbol{v}^{\top}\mathfrak{n}_{\mathrm{c}}(C_t)^{\top}\mathcal{V}_n^t(Y)^{\top}\mathcal{V}_n^t(Y)\mathfrak{n}_{\mathrm{c}}(C_t)\boldsymbol{v}.
\end{align*}
Note that $\mathfrak{n}_{\mathrm{c}}(C_t)\boldsymbol{v}$ is the coefficient vector of $g$. Thus, if $\mathcal{V}_n^t(Y)^{\top}\mathcal{V}_n^t(Y)=I$, then we can estimate the coefficient norm of $g$ from the random evaluation vector $g(Y)$. However, this cannot be achieved. For instance, when $Y=\{(y_1, y_2)^{\top}\} \subset \mathbb{R}^2$,
\begin{align*}
    \mathcal{V}_2^2(Y)^{\top}\mathcal{V}_2^2(Y) &= \left( \begin{array}{cccc}{1} & {y_1}  \cdots & y_2^2 \end{array}\right)^{\top}\left( \begin{array}{cccc}{1} & {y_1}  \cdots & y_2^2 \end{array}\right),\\
    &= \left( \begin{array}{ccccc}
1 &  &  &  & \\  
  & y_1^2 &  &  & \\
  &  & \ddots &  & y_1^2y_2^2\\
  &  &    & y_1^2y_2^2  & \\
  &  &    &   & \ddots \\
\end{array}\right),
\end{align*}
which has $y_1^2y_2^2$ both in diagonal and off-diagonal entries. Therefore, $\mathcal{V}_2^2(\boldsymbol{x})^{\top}\mathcal{V}_2^2(\boldsymbol{x})$ will not be the identity matrix regardless of the sampled points. 

\section{Experiments}
\begin{figure*}[t]
\includegraphics[width=\linewidth]{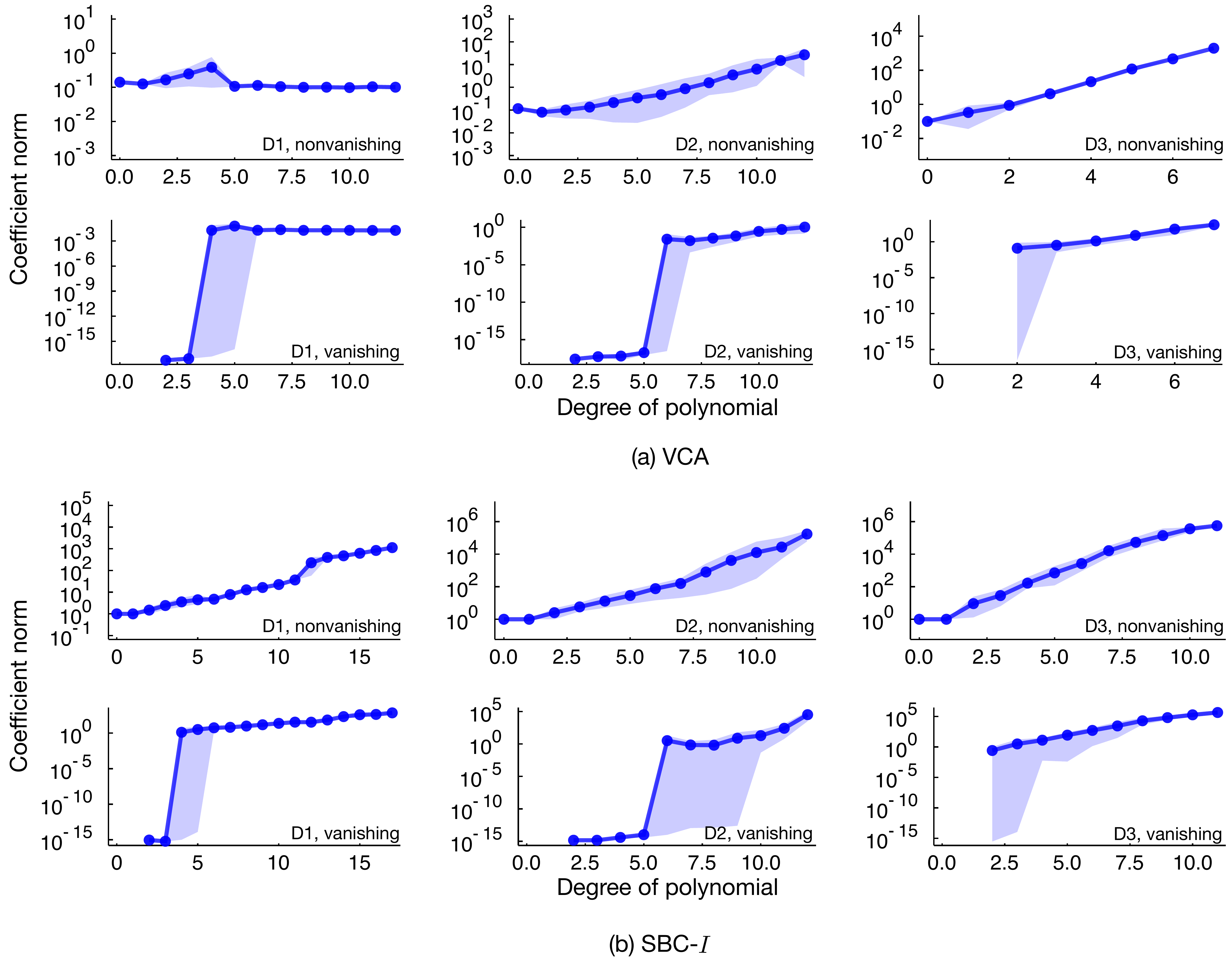}\caption{Coefficient norm of nonvanishing polynomials (upper row of each panel) and vanishing
polynomials (lower row of each panel) by (a) VCA and (b) SBC-$I$ for three datasets (each column for
each dataset). The mean coefficient norms of each degree are linked by
solid lines. The range from the smallest to the largest coefficient norms is represented by shades. The coefficient norm is considerably different even at a degree, and the average coefficient norm increases sharply over degree.}
\label{fig:vca-coef}
\end{figure*}
\begin{figure*}[t]
\includegraphics[width=\linewidth]{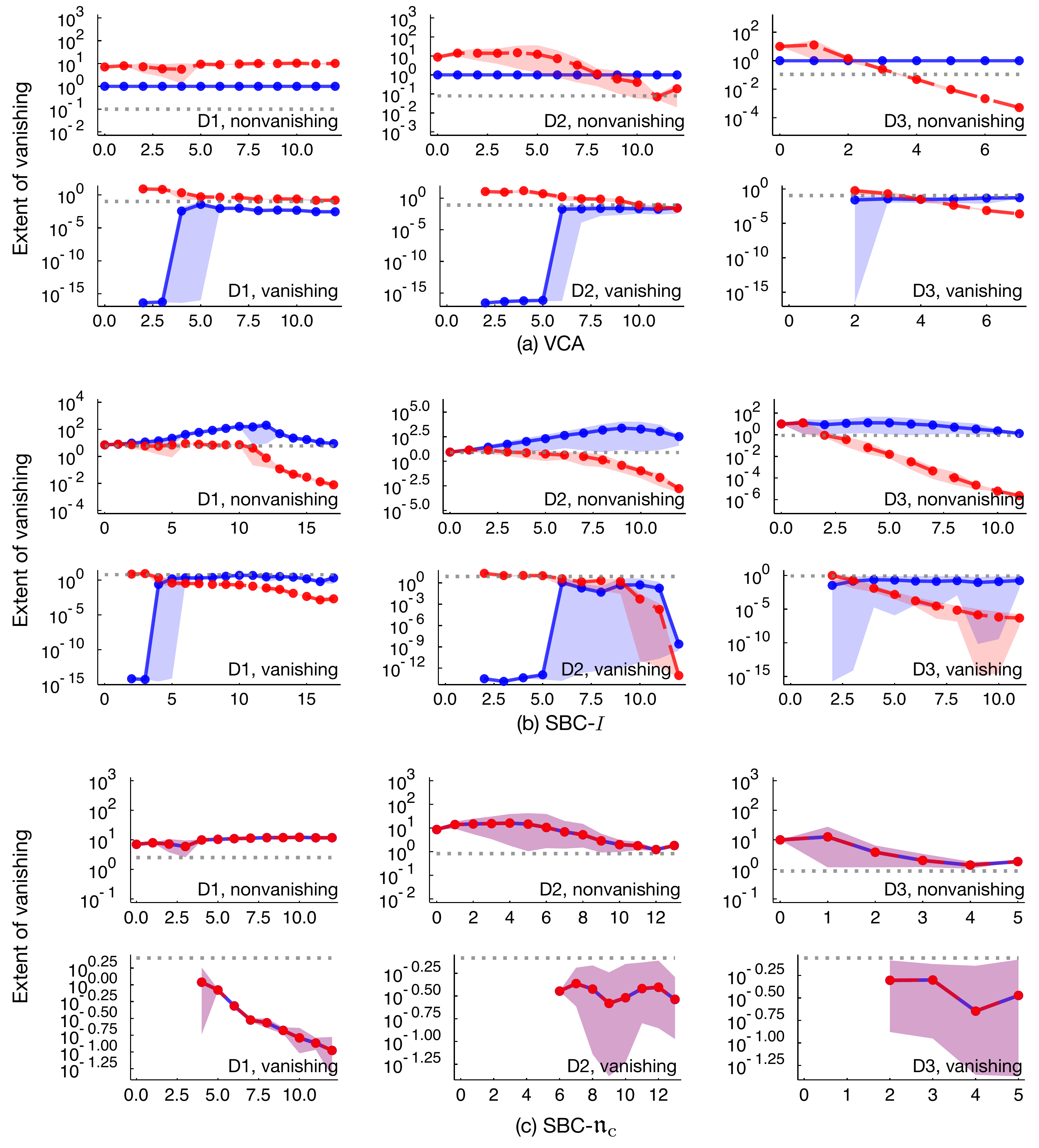}\caption{The extent of vanishing of polynomials obtained by VCA (a) and SBC-$I$ (b), and  SBC-$\mathfrak{n}_{\mathrm{c}}$ (c) for three datasets (each column
for each dataset). In each panel, the upper and lower rows show the result for nonvanishing and vanishing polynomials, respectively. The mean extent of vanishing at each degree is linked by dots and solid lines (blue), and the mean extent of vanishing of polynomials whose coefficient norm is normalized to unity by a post-processing is linked by dots and dashed lines (red). 
The range from the smallest to the largest extent of vanishing is represented by shades (red and blue). Dotted lines (gray) represent threshold $\epsilon$. (a, b) using VCA and SBC-$I$, some nonvanishing polynomials turn into vanishing polynomials when the coefficient vectors are normalized, whereas some vanishing polynomials turn into nonvanishing polynomials. (c) by taking the coefficient normalization into account (SBC-$\mathfrak{n}_{\mathrm{c}}$), the extent of vanishing remains invariant when coefficients are normalized because the coefficients
are already normalized during the basis construction. }
\label{fig:vanishingness}
\end{figure*}
\begin{table*}[t]
\caption{Statistics of basis sets of nonvanishing polynomials computed by SBC-$I$ and SBC-$\mathfrak{n}_{\mathrm{c}}$ with coefficient truncation thresholds $\theta=0.0,0.5,0.9,1.0$. The basis sets of degree up to ten are considered in the calculation of the statistics for fair comparison. One can see that (i) the nonvanishing polynomials obtained by SBC-$I$ have significantly large coefficient norms, while those found by SBC-$\mathfrak{n}_{\mathrm{c}}$ ($\theta=1.0$) have unit coefficient norms; (ii) even with the coefficient truncation ($\theta=0.0,0.5,0.9$), the coefficient norms are close to unity despite of the drastically shortened coefficient vectors; and (iii) at the same time, runtimes and memories are reduced by approximately 2--20 times.}\label{table:coefficient-approximation}
\centering
\small
\begin{tabular}{|c|c|c|c|c|c|c|}
\hline 
 &  & \multicolumn{4}{c|}{SBC-$\mathfrak{n}_{\mathrm{c}}$} & SBC-$I$\tabularnewline
\hline 
 & $\theta$ & 0.0 & 0.5 & 0.9 & 1.0 & \tabularnewline
\hline 
\hline 
\multirow{5}{*}{D2$^{+}$} & Length of coeff. vec. & 21 & 21 & 233 & 19427 & 19448\tabularnewline
\cline{2-7} 
 & Mean coeff. norm & 2.11 & 2.11 & 3.00 & 1.00 & 1.19e+4\tabularnewline
\cline{2-7} 
 & Min / Max coeff. norm & 0.60 / 4.46 & 0.60 / 4.46 & 1.00 / 6.21 & 1.00 / 1.00 & 1.00 / 2.33e+5\tabularnewline
\cline{2-7} 
 & Runtime {[}msec{]} & 1.50e+1 & 1.52e+1 & 1.62e+2 & 4.58e+1 & 1.57\tabularnewline
\cline{2-7} 
 & Memory {[}MB{]} & 2.46e+1 & 2.46e+1 & 2.72e+1 & 8.18e+1 & 2.42\tabularnewline
\hline 
\multirow{5}{*}{D3$^{+}$} & Length of coeff. vec. & 38 & 38 & 4207 & 646625 & 646646\tabularnewline
\cline{2-7} 
 & Mean coeff. norm & 1.79 & 1.79 & 2.12 & 1.00 & 1.08e+8\tabularnewline
\cline{2-7} 
 & Min / Max coeff. norm & 0.54 / 5.30 & 0.54 / 5.30 & 0.88 / 3.89 & 1.00 / 1.00 & 1.00 / 1.00e+10\tabularnewline
\cline{2-7} 
 & Runtime {[}msec{]} & 8.17e+2 & 8.31e+2 & 3.61e+3 & 1.79e+4 & 4.47e+1\tabularnewline
\cline{2-7} 
 & Memory {[}MB{]} & 1.10e+3 & 1.1e+3 & 4.01e+3 & 1.17e+4 & 3.54\tabularnewline
\hline 
\end{tabular}
\end{table*}

\begin{table}[t]
\caption{Classification results by VCA and SBC-$\mathfrak{n}_\mathrm{c}$ in three datasets (Iris, Vowel, and Vehicle). SBC-$\mathfrak{n}_{\mathrm{c}}$ achieved comparable or even lower errors than VCA with significantly shorter feature vectors, which implies VCA basis sets contain many spurious vanishing polynomials. The results are averaged over ten independent runs.}~\label{table:classification-result}
\centering
\begin{tabular}{|c|c|c|c|c|c|}
\hline 
 &  & \multicolumn{3}{c|}{SBC-$\mathfrak{n}_{\mathrm{c}}$} & VCA\tabularnewline
\hline 
 & $\theta$ & 0.5 & 0.9 & 1.0 & \tabularnewline
\hline 
\hline 
\multirow{4}{*}{Ir.} & Error & 0.03 & 0.05 & 0.06 & 0.05\tabularnewline
\cline{2-6} 
 & Length of $\mathcal{F}(\cdot)$ & 8.20e+1 & 9.29e+1 & 6.35e+1 & 1.51e+2\tabularnewline
\cline{2-6} 
 & Runtime {[}msec{]} & 1.01e+1 & 1.10e+1 & 8.66 & 6.90\tabularnewline
\cline{2-6} 
 & Memory {[}MB{]} & 5.61 & 6.50 & 4.91 & 4.74\tabularnewline
\hline 
\multirow{4}{*}{Vo.} & Error & 0.45 & 0.50 & 0.34 & 0.45\tabularnewline
\cline{2-6} 
 & Length of $\mathcal{F}(\cdot)$ & 3.29e+3 & 3.24e+3 & 3.12+3 & 4.74e+3\tabularnewline
\cline{2-6} 
 & Runtime {[}msec{]} & 5.47e+3 & 5.85e+3 & 4.89e+3 & 7.20e+3\tabularnewline
\cline{2-6} 
 & Memory {[}MB{]} & 5.24e+2 & 5.71e+2 & 6.88e+2 & 3.97e+2\tabularnewline
\hline 
\multirow{4}{*}{Ve.} & Error & 0.26 & 0.25 & 0.20 & 0.19\tabularnewline
\cline{2-6} 
 & Length of $\mathcal{F}(\cdot)$ & 5.57e+3 & 5.78e+3 & 5.25e+3 & 8.26e+3\tabularnewline
\cline{2-6} 
 & Runtime {[}msec{]} & 2.95e+4 & 3.89e+4 & 4.69e+4 & 1.31e+4\tabularnewline
\cline{2-6} 
 & Memory {[}MB{]} & 4.23e+3 & 4.24e+3 & 4.38e+3 & 1.01e+3\tabularnewline
\hline 
\end{tabular}
\end{table}

Here, we compare VCA, SBC without any normalization (i.e., SBC with \texttt{Step2}; SBC-$I$), and SBC with the coefficient normalization (SBC-$\mathfrak{n}_{\mathrm{c}}$). 
In the first experiment, we show that VCA and SBC-$I$ encounter severe spurious vanishing problem even in simple datasets, whereas SBC-$\mathfrak{n}_{\mathrm{c}}$ does not. 
In the second experiment, we compare these methods in classification tasks. All experiments were performed using Julia implementation on a desktop machine with an eight-core processor and a 32 GB memory. We emphasize that the proposed methods (coefficient normalization with the generalized eigenvalue problem and the coefficient truncation) can be easily unified with other basis construction methods because these methods are all based on the SBC framework. However, these methods are less commonly used than VCA, and they need more hyperparameters to control, which makes the analysis unnecessarily complicated.

\subsection{Analysis of Coefficient Norm and the Extent of Vanishing with Simple Datasets}\label{sub:Results1}
We perform basis construction by VCA, SBC-$I$, and SBC-$\mathfrak{n}_{\mathrm{c}}$. The coefficient norm and the extent of vanishing of obtained polynomials are respectively compared between three methods. We also compare SBC-$\mathfrak{n}_{\mathrm{c}}$ with and without the coefficient truncation. 

\paragraph*{Datasets and parameters}
We use three algebraic varieties: (D1) double concentric circles (radii
1 and 2), (D2) triple concentric ellipses (radii $(\sqrt{2}, 1/\sqrt{2})$, $(2\sqrt{2}, 2/\sqrt{2})$, and $(3\sqrt{2}, 3/\sqrt{2})$) with $3\pi/4$
rotation, and (D3) $\left\{ xz-y^{2},x^{3}-yz\right\} $. We randomly
sampled 50, 70, and 100 points from these algebraic varieties, respectively. 
We further consider two datasets by adding variables to D2 and D3. (D2$^+$) five additional variables $y_i = k_ix_1+(1-k_i)x_2$ for $k_i\in\{0.0, 0.2, 0.5, 0.8, 1.0\}$, where $x_1$ and $x_2$ are the variables of D2. (D3$^+$) nine additional variables $y_i = k_ix_1+l_ix_2+(1-k_i-l_i)x_3$ for $(k_i,l_i)\in \{0.2,0.5,0.8\}^2$, where $x_1$, $x_2$, and $x_3$ are the variables of D3. Each dataset  is mean-centralized and then perturbed by the additive Gaussian noise. The mean of the noise is set to zero, and the standard deviation is set to 5\% of the average absolute value of the points. 

For each dataset and method, the threshold $\epsilon$ is selected as follows. First, we compute a Gr\"obner basis $G$ of the algebraic variety of the dataset. Suppose $G$ contains $M_{t}, M_{t+1}, \cdots, M_{T}$ polynomials at degree $t, t+1, ..., T$, respectively, where $t$ and $T$ are the lowest degree and highest degree of polynomials in $G$, respectively. Then, $\epsilon$ is selected so that the target basis construction yields a basis $G^{\prime}$ whose lowest-degree polynomial is degree $t$, and $|G^{\prime}_{\tau}| \ge M_{\tau}$ for $\tau \ge t$. To be more precise, we first search the range of such thresholds, and set $\epsilon$ to the mean of that range.

\paragraph*{Results}
In Fig.~\ref{fig:vca-coef}, the coefficient norm of nonvanishing polynomials (upper row of each panel) and that of vanishing polynomials (bottom row of each panel), which are obtained by VCA and SBC-$I$, are plotted along the degree. 
The mean values are represented by solid lines and dots, and the range from minimum to maximum is represented by shades. As can be seen from the figure, the mean coefficient norm tends to sharply grow along the degree (note that the vertical axes are in the logarithm scale) for both methods. Even within a degree, there can be a huge gap as in degree-5 VCA vanishing polynomials of D1, degree-6 SBC-$I$ vanishing polynomials of D2 (bottom middle panel), and so on. 
These results imply that some vanishing (or nonvanishing) polynomials might be vanishing (or nonvanishing) merely due to their small (or large) coefficients; such polynomials might become nonvanishing (or vanishing) polynomials once these polynomials are normalized to have a unit coefficient norm. In fact, this is corroborated by the result shown in Fig.~\ref{fig:vanishingness}(a,b). The extent of vanishing (blue dots and solid lines) is contrasted against the rescaled extent of vanishing (red dots and dashed lines), which is calculated by rescaling the extent of vanishing using the coefficient norm of polynomials at post-processing so that polynomials have a unit coefficient norm. 
After the rescaling, some nonvanishing polynomials show the extent of vanishing below the threshold (gray dotted line) and some vanishing polynomials show the extent of vanishing above the threshold. For example, degree-$5$ VCA vanishing polynomials become nonvanishing polynomials after the rescaling; degree-10 SBC-$I$ nonvanishing polynomials become vanishing polynomials after the rescaling. The variance of the extent of vanishing at each degree also changes drastically. For example, the rescaling degree-5 VCA vanishing polynomials show large variance of the extent of vanishing, but the rescaling reveals that the actual extent of vanishing is almost identical to these polynomials. The reverse is also observed as in degree-5 VCA nonvanishing polynomials.
Note that both VCA and SBC-$I$ required expensive calculations for this post-processing (rescaling), because usually, these methods cannot access the coefficient norm of polynomials (especially in the numerical implementation).

In contrast, as shown in Fig.~\ref{fig:vanishingness}(c), the extent of vanishing of polynomials from SBC-$\mathfrak{n}_{\mathrm{c}}$ are consistent before and after the normalization, which is simply because the polynomials are generated under the coefficient normalization. Moreover, we can see that under coefficient normalization, the extent of vanishing shows considerably lower variance for both nonvanishing and vanishing polynomials. In other words, VCA and SBC-$I$ overestimate (or underestimate) the extent of vanishing due to the bloat in the coefficient norm.

Next, we evaluate SBC-$\mathfrak{n}_{\mathrm{c}}$ with the coefficient truncation. The result is summarized in Table~\ref{table:coefficient-approximation}. We change the truncation threshold $\theta$ in~(\ref{eq:truncation-threshold}) from 0.0 to 1.0. Following Theorem~\ref{thm:truncation-limit}, we keep at least $|F_t|$ coefficients at each degree regardless of $\theta$. Thus, $\theta=0.0$ corresponds to the case where we keep exactly $|F_t|$ coefficients for each degree. $\theta=1.0$ corresponds to SBC-$\mathfrak{n}_{\mathrm{c}}$ without the coefficient truncation.
Here, we analyze the nonvanishing polynomials in terms of the length of coefficient vectors, the actual coefficient norm (mean, minimum, and maximum), the runtime of basis construction, and the memory used during the basis construction. To measure these statistics consistently across methods and parameters, the basis construction is terminated at degree 10 even if the termination condition is not satisfied.
We also show the same statistics of SBC-$I$. As for VCA, we cannot find proper parameter $\epsilon$ so that the degree and number of basis polynomials are similar to the Gr\"obner basis. VCA rescales each nonvanishing polynomial by the norm of its evaluation vector during the basis construction. We consider that this rescaling can lead to more spurious vanishing polynomials, resulting in too early termination, or lead to more spurious nonvanishing polynomials, resulting in fewer vanishing polynomials than those of the Gr\"obner basis at each degree. Because the computation of VCA and SBC-$I$ is quite similar, it is enough only to consider SBC-$I$ for measuring runtime and memory.

As can be seen in Table~\ref{table:coefficient-approximation},
with $\theta=0.9$, the truncated coefficient vectors are approximately 100 times shorter. Nevertheless, the mean, minimum, and maximum of the coefficient norm are still moderately close to unity, respectively, for both datasets. This means that only about 1\% of monomials and coefficients have a significant contribution to the basis polynomials. Even in the extreme case ($\theta=0.0$), the coefficient norm of polynomials still lies in the moderate range, while the coefficient vectors are significantly shortened (less than 0.1\%). 
By the coefficient truncation, the runtime and memory for SBC-$\mathfrak{n}_{\mathrm{c}}$ is reduced. For example, at $\theta=0.9$, the runtime and memory of SBC-$\mathfrak{n}_{\mathrm{c}}$ is reduced by around 3 for both datasets; at $\theta=0.5$, the runtime is reduced by 20 times for D3$^{+}$.
SBC-$I$ remains faster than SBC-$\mathfrak{n}_{\mathrm{c}}$ even with $\theta=0.0$. However, the coefficient norm of SBC-$I$ significantly varies across polynomials (e.g., $10^{10}$ gap between minimum and maximum for D3$^+$). In other words, the fast calculation of SBC-$I$ is a consequence of allowing the basis construction to encounter the spurious vanishing problem.

Again, note that coefficient vectors are typically not accessible for VCA and SBC-$I$ in the numerical implementation. Thus, one cannot normalize nor discard polynomials by weighing their coefficient norms, as done in the analysis. For the above analysis, we calculated the coefficient vectors for VCA and SBC-$I$ in the same manner as in SBC-$\mathfrak{n}_{\mathrm{c}}$, which takes the additional cost. The runtime was measured by independently running VCA and SBC-$I$ without the coefficient calculation. 

\subsection{Classification\label{sec:classification}}
Here, we extract feature vectors from data using vanishing polynomials, and train a linear classifier with these vectors. We compare VCA and SBC-$\mathfrak{n}_{\mathrm{c}}$ with and without the coefficient truncation. In the training stage, we compute vanishing polynomials for each class. Let $G_{i}=\{g^{(i)}_{1},...,g^{(i)}_{|G_{i}|}\}$ be the vanishing polynomials of the $i$-th class data. 
As proposed in~\cite{livni2013vanishing}, a feature vector of a data point $\boldsymbol{x}$ is given by
\begin{align*}
\mathcal{F}(\boldsymbol{x}) & =\Bigl(\cdots,\underbrace{\left|g^{(i)}_{1}(\boldsymbol{x})\right|,\cdots,\left|g^{(i)}_{|G_{i}|}(\boldsymbol{x})\right|}_{|G_{i}(\boldsymbol{x})|^{\top}},\cdots\Bigr)^{\top},
\end{align*}
where $g_j^{(i)}$ is the $j$-th vanishing polynomials of the $i$-th class.
Intuitively, $\mathcal{F}(\boldsymbol{x}^{(i)})$ for the $i$-th class data point $\boldsymbol{x}^{(i)}$ takes small values for $G_i$ part and large values for the rest. Note that these feature vectors do not exploit class-discriminative information because the basis set for each class is independently constructed by only using data of the corresponding class.
We employed $\ell_2$-regularized logistic regression and one-versus-the-rest strategy using LIBLINEAR~\cite{rong09liblinear}. 

\paragraph*{Datasets and paramters}
We used three datasets (Iris, Vowel, and Vehicle) from the UCI dataset repository~\cite{Lichman2013machine}. The parameter $\epsilon$ was selected by 3-fold cross-validation. Because Iris and Vehicle do not have prespecified training sets and test sets, we randomly split each dataset into a training set (60\%) and test set (40\%), which were mean-centralized and normalized so that the mean norm of data points is equal to unity.

\paragraph*{Results}
As can be seen from Table~\ref{table:classification-result}, SBC-$\mathfrak{n}_{\mathrm{c}}$ achieves comparable or lower classification error than VCA for all the datasets. In particular, the improvement at the Vowel dataset is significant. Note that SBC-$\mathfrak{n}_{\mathrm{c}}$ yields much shorter feature vectors than VCA. For example, in the Vehicle dataset, SBC-$\mathfrak{n}_{\mathrm{c}}$ feature vectors are shorter than the VCA feature vectors by approximately 3,000 (about 36\% reduction). Furthermore, in contrast to the previous experiment using D2$^+$ and D3$^+$, the gap of runtimes between SBC-$\mathfrak{n}_{\mathrm{c}}$ and VCA is much less significant. This is because the termination degree of the basis construction is rather low (approximately five, in most cases). In such a case, the length of coefficient vectors is not the bottleneck of the runtime. This also explains why the coefficient truncation not necessarily decreases the runtime of SBC-$\mathfrak{n}_{\mathrm{c}}$. It also can be seen that the effect of the coefficient truncation on the classification error is not consistent. The classification error decreases for the Iris dataset and increases for the other datasets. To pursue better performance in classification, it is necessary to consider class-discriminative information in basis construction and coefficient truncation. For example, introducing coefficient normalization in discriminative VCA (DVCA;~\cite{hou2016discriminative}) is an interesting future work. 

\section{Conclusion and Future Work}
In this paper, we discussed the spurious vanishing problem in the approximate vanishing ideal, which has been an unnoticed theoretical flaw of existing polynomial-based basis constructions. To circumvent the spurious vanishing problem, polynomial-based basis constructions are required to introduce a normalization. We propose a method to optimally generate basis polynomials under a given normalization. The proposed method is enough general to extend the existing basis construction algorithms and to consider various types of normalization. In particular, we consider coefficient normalization, which is intuitive but costly to introduce. We propose two methods to ease the computational cost; one is an exact method that takes advantages of the iterative nature of the basis construction framework, and the other is an approximation method, which empirically but drastically shortens the coefficient vectors while keeping the coefficient norm of the polynomials in a moderate range.

The experiments show the severity of the spurious vanishing problem in basis construction algorithms without proper normalization (VCA and SBC-$I$) and the effectiveness of the proposed method for avoiding the problem. In the classification tasks, SBC with coefficient normalization achieved comparable or even lower classification errors with much shorter feature vectors than unnormalized methods.
An important future direction is to design a more scalable algorithm. Our experiments suggest that the coefficient norm of polynomials is well regularized even when only a few proportions of monomials are considered. This can be a key observation to reduce the runtime of new algorithms. Another interesting direction is to consider a different type of normalization.

\section*{Acknowledgement}
This work was supported by JSPS KAKENHI Grant Number 17J07510.

\FloatBarrier

\end{document}